\documentclass[journal]{IEEEtran}
\usepackage[utf8]{inputenc}
\usepackage[noadjust]{cite}
\usepackage[margin=1in]{geometry} 
\usepackage{amsmath,amsthm,amssymb}
\usepackage{pdfpages}
\usepackage{balance}
\usepackage{bm}
\usepackage{graphicx}
\usepackage{yfonts}
\usepackage{xcolor}
\usepackage{booktabs}
\usepackage{siunitx}
\newcommand{\ra}[1]{\renewcommand{\arraystretch}{#1}}
\usepackage[caption=false]{subfig}

\usepackage{svg}
\newcommand{\orcid}[1]{\href{https://orcid.org/#1}{\includegraphics[width=10pt]{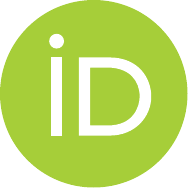}}}

\usepackage{hyperref}
\hypersetup{
    colorlinks=true,
    citecolor=blue,
    urlcolor=cyan,
}

\theoremstyle{definition}
\newtheorem{theorem}{Theorem}
\newtheorem{lemma}{Lemma}
\newtheorem{definition}{Definition}
\newtheorem{problem}{Problem}
\DeclareMathOperator{\tr}{tr}
\DeclareMathOperator{\sgn}{sgn}
\DeclareMathOperator{\diag}{diag}

\title{Decentralized Adaptive Control for Collaborative Manipulation of Rigid Bodies}
\author{Preston Culbertson\orcid{0000-0002-1403-8697},
Jean-Jacques Slotine\orcid{0000-0002-7161-7812},
and Mac Schwager\orcid{0000-0002-7871-3663}% <-this % stops a space
\thanks{P. Culbertson is with the Department of Mechanical Engineering, Stanford University, \texttt{pculbertson@stanford.edu}.}
\thanks{J.-J. Slotine is with the Department of Mechanical Engineering and the Department of Brain and Cognitive Sciences, Massachusetts Institute of Technology, \texttt{jjs@mit.edu}.}
\thanks{M. Schwager is with the Department of Aeronautics and Astronautics, Stanford University, \texttt{schwager@stanford.edu}.}
\thanks{This work was supported in part by the NASA Space Technology Research Fellowship Grant 80NSSC18K1180.}
\thanks{This paper will appear in the IEEE Transactions on Robotics. ©~2021 IEEE.  Personal use of this material is permitted.  Permission from IEEE must be obtained for all other uses, in any current or future media, including reprinting/republishing this material for advertising or promotional purposes, creating new collective works, for resale or redistribution to servers or lists, or reuse of any copyrighted component of this work in other works.}
}

\begin{document}
\maketitle
\begin{abstract}
    In this work, we consider a group of robots working together to manipulate a rigid object to track a desired trajectory in $SE(3)$.  The robots do not know the mass or friction properties of the object, or where they are attached to the object. They can, however, access a common state measurement, either from one robot broadcasting its measurements to the team, or by all robots communicating and averaging their state measurements to estimate the state of their centroid. To solve this problem, we propose a decentralized adaptive control scheme wherein each agent maintains and adapts its own estimate of the object parameters in order to track a reference trajectory. We present an analysis of the controller's behavior, and show that all closed-loop signals remain bounded, and that the system trajectory will almost always (except for initial conditions on a set of measure zero) converge to the desired trajectory. We study the proposed controller’s performance using numerical simulations of a manipulation task in 3D, as well as hardware experiments which demonstrate our algorithm on a planar manipulation task. These studies, taken together, demonstrate the effectiveness of the proposed controller even in the presence of numerous unmodeled effects, such as discretization errors and complex frictional interactions.
\end{abstract}
\vspace{-1em}
\vspace{-1em}

\section{Introduction}
As robots are fielded in applications ranging from disaster relief to autonomous construction, it is likely that such applications will, at times, require manipulation of objects which exceed a single robot's actuation capacities (e.g. debris removal, material transport) \cite{Khatib:1999aa}. Notably, when humans undertake such tasks, they are able to work together to flexibly and reliably transport objects which are too massive for one person alone. We are interested in enabling similar capabilities in robot teams.

In this work, we consider the problem of cooperative manipulation, where a group of robots works together to manipulate a common payload. This problem has long been of significant interest to both the multi-robot systems community, and the broader robotics community \cite{Khatib:1996ab}, since it requires close collaboration between robots in addition to the numerous sensing and actuation challenges involved in any manipulation task.

While collaborative manipulation has been studied somewhat extensively in literature \cite{Khatib:1996aa,Fink:2008aa,Rus:1995aa}, manipulation teams are still quite limited in their functionality, often due to restrictive assumptions. Many methods require exact payload knowledge (i.e., the payload's mass, inertial properties, and the location of its center of mass), which imposes the requirement that every payload be thoroughly calibrated prior to manipulation, in addition to making manipulation teams brittle in the face of parameter uncertainty. Further, many algorithms require a centralized planner, which introduces both latency and a single point of failure to the manipulation system. Recent works have proposed handling payload uncertainty by explicitly estimating physical parameters before manipulation \cite{ Corah:2017aa, Franchi:2019aa}, or via an indirect scheme which combines a consensus-based estimator with a certainty-equivalent controller \cite{Lee:2017aa}. We consider a different approach, decentralized adaptive control, where each agent uses a single adaptive controller to perform estimation and control simultaneously.

In this work, we take a minimalist approach to collaborative manipulation, and present a control scheme which requires no prior payload characterization and can be computed locally on board each agent. To achieve this, we use a controller on $SE(3)$, together with a decentralized adaptation law, which allows the manipulation team to adapt to unknown payload dynamics; we also prove our controller's convergence to tracking a reference trajectory using Barbalat's lemma. We only require that the agents share a common state measurement, which can be achieved by having one agent broadcast its measurements to the entire manipulation team, or instead by having agents communicate and average their state measurements to estimate the state of their centroid. This estimate could also be obtained by running a consensus algorithm \cite{Olfati-Saber:2004aa} online, using a mesh network between agents.

Intuitively, during manipulation each agent executes the proposed controller, with its individual set of parameter estimates (or, equivalently, control gains). The agents then use the common measurement to adapt their local parameters, such that the measurement point asymptotically tracks a desired trajectory. Such an approach leads to a manipulation algorithm which can easily scale in the number of agents, and is flexible enough to manipulate a large variety of payloads, while tolerating both parameter uncertainty and single-agent failure.

This paper proposes control and adaptation laws for manipulation tasks in both the plane and three-dimensional space, in addition to a proof of convergence to the desired trajectory, as well as the boundedness of all closed-loop signals. We further validate the controller's performance both numerically, simulating a manipulation task in 3D, and experimentally, using a team of ground robots performing a planar manipulation task. %\textcolor{red}{Mac: Careful, stability and boundedness of signals are two different things.}

\section{Prior Work}
Collaborative manipulation has been a subject of both sustained and broad interest in multi-robot systems. Early work such as \cite{Alford:1984aa} first proposed multi-arm manipulation strategies which used a central computer to control each arm. This early work was followed with a flurry of interest, including \cite{Khatib:1988aa,Khatib:1996aa}, where the authors study the dynamics and force allocation of redundant manipulators, and \cite{Wen:1992aa}, which investigates various control strategies for multiple arms. A number of authors \cite{Bonitz:1996aa, Caccavale:2008aa, Erhart:2013aa} have used impedance controllers for collaborative manipulation, in order to yield compliant object behavior, as well as to avoid large internal forces on the payload. An approach for characterizing such forces is proposed in \cite{Erhart:2015aa}.

While early work focused almost exclusively on multi-arm manipulation,  often with centralized architectures and only a few robots, other authors have focused instead on collaborative manipulation with large teams of mobile robots. In \cite{Rus:1995aa} and \cite{Bohringer:1997aa}, the authors move furniture with small mobile robots, with a focus on reducing the communication and sensing requirements of manipulation algorithms.  Other works have studied various decentralized manipulation strategies, including potential fields \cite{Song:2002aa}, caging \cite{Fink:2008aa}, and distributed control allocation via matrix pseduo-inverse \cite{Faal:2016aa}. In \cite{Mellinger:2013aa}, the authors consider the problem of collaborative lifting using quadrotors which rigidly grasp a common payload; they use a least-squares solution to allocate control efforts between agents.

While most work, including this paper, assumes the robots are rigidly attached to the payload, some authors have studied the problem under different grasping assumptions. \cite{Lynch:1999aa} studies the problem of collaborative manipulation with thrusters (i.e., agents can only apply force to the body along a fixed line), and investigates the minimum number of thrusters needed to achieve nonlinear controllability of the system. Other authors have studied aerial towing (manipulation via cables) in works such as \cite{Sreenath:2013aa, Tang:2020aa, Michael:2011aa}. The authors in \cite{Jackson:2020aa} use a distributed optimization algorithm to perform collaborative manipulation of a cable-suspended load. 

% \textcolor{red}{Please include my 2020 RA-L paper with Zac, Kunal, Brian, and Adam.  We use ADMM for collaborative cable suspended manipulation.}. 

While collaborative manipulation has been studied in a variety of settings, many works rely on restrictive assumptions, such as central planning or prior knowledge of the object parameters, which limit the practical applicability of their approaches. More recent work has focused on removing these assumptions. In \cite{Wang:2016aa}, the authors propose a controller which allows a group of agents to reach a force consensus without communication, using only the payload dynamics; a leader can steer this consensus to have the object track a desired trajectory. Further, \cite{Wang:2018aa} proposes a communication-free controller for collaborative manipulation with aerial robots, using online optimization to allocate control effort. \cite{Corah:2017aa} proposes a centralized estimation scheme, wherein the robots use an active perception strategy to minimize the uncertainty of their estimates of various payload parameters. In \cite{Franchi:2019aa}, \cite{Petitti:2016aa} the authors propose to first characterize the payload via an explicit distributed estimation scheme, in which robots use a communication network and local measurements to estimate both the mass and geometric properties of their common payload, before using these estimates to perform manipulation.

Previous work has proposed applying adaptive control to collaborative manipulation. The authors of \cite{Hu:1995aa} propose a centralized adaptive control scheme, which they demonstrate experimentally with a pair of robot arms. In \cite{Li:2007aa}, the authors propose a centralized robust adaptive controller for a group of mobile manipulators performing a collaborative manipulation task. 

% \textcolor{red}{Mac: More detail here on the differences between ours and theirs.} 

Perhaps closest to this work are \cite{Liu:1998aa, Verginis:2017aa}, which also propose distributed adaptive control schemes for collaborative manipulation. This work differs from previous literature in adapting both inertial \textit{and} geometric parameters of the body; previous methods require exact measurement of manipulators' relative positions from the center of mass in order to invert the grasp matrix \cite{Prattichizzo:2008aa}, which relates wrenches applied at contact points to wrenches applied at the center of mass. We argue that this assumption is quite restrictive, since, in practice, this amounts to knowing exactly how the body's translational and rotational dynamics are coupled. Accurately localizing large numbers of manipulators (especially with respect to the center of mass) can be both time consuming and difficult; it is also unclear why the body's normalized first moment of mass (i.e., the center of mass) is assumed known, while its zeroth (mass) and second moments (inertia tensor) are assumed unknown. In contrast, our method can treat a much broader class of model parameters (including inertial, geometric, and other effects such as gravity and friction) as unknowns, and handle them through online adaptation. We believe that this provides greater robustness to uncertainty in these parameters, and eliminates the need to estimate these quantities before manipulation.

An early version of this work appeared in \cite{Culbertson:2018aa}; this work differs greatly from the original version. The most fundamental difference is that this work builds on the Slotine-Li adaptive controller \cite{Slotine:1987aa}, which allows for much more efficient parameterization than the general model reference adaptive controller in the original version. We do not constrain the measurement point to lie at the center of mass, and further introduce a composite error that allows for trajectory tracking instead of simple velocity control. Finally, we present new simulation and experimental results studying the novel controller.

Our basic algorithm can also be compared to quorum sensing \cite{Tabareau:2010aa,Russo:2010aa}, a natural phenomenon in which organisms or cells can exchange information~\cite{Waters:2005aa} and synchronize themselves using a noisy, shared environment, instead of using direct information exchange. Such mechanisms have been shown to reduce the influence of random noise on synchronization phenomena \cite{Tabareau:2010aa}, allowing for meaningful coordination without centralized control or explicit communication. The proposed algorithm can be seen as analogous to quorum sensing, wherein the common object acts as the shared environment, which allows the agents to coordinate their control actions to achieve tracking.

\section{Problem Statement}
We consider a rigid body which undergoes general motion (i.e., translation and rotation), during collaborative manipulation by a team of $N$ autonomous agents. There exists a world-fixed (Newtonian) frame $\mathcal{F}$, as well as a body-fixed frame $\mathcal{B}$, which is centered at $\mathbf{b}_{cm}$, the body's center of mass. There also exists a body-fixed measurement point $P$ which is located at $\mathbf{r}_p$ in $\mathcal{B}$. The body has mass $m$ and body-frame inertia of $\mathbf{I}_{cm}$ about its center of mass. Each agent is rigidly attached to the body at $\mathbf{r}_i$ from $P$, and applies a wrench $\bm{\tau}_i$, i.e., a combined force and torque. Figure \ref{fig:schem} shows a schematic of the system.

\begin{figure}[ht!]
    \centering
    \includegraphics[width=0.4\textwidth]{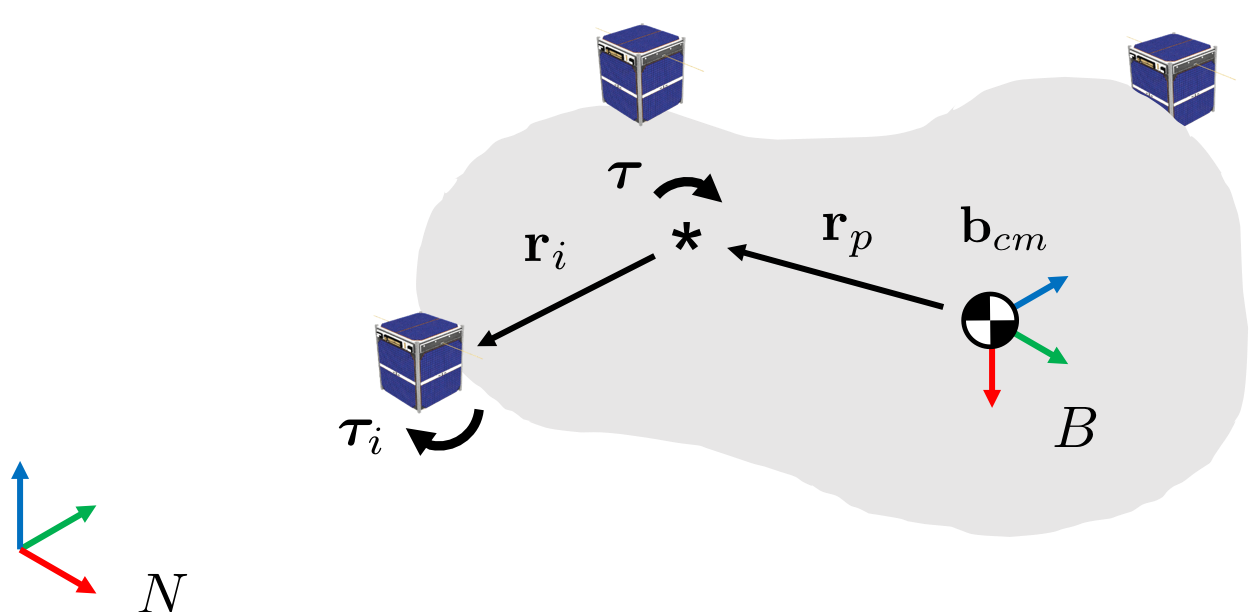}
    \caption{Schematic of the collaborative manipulation task.}
    \label{fig:schem}
\end{figure}

We describe the body's pose in $\mathcal{F}$ using a configuration variable $\mathbf{q} = (\mathbf{x}, \mathbf{R}) \in \mathbb{R}^n\times SO(n)$, where $$SO(n) = \left\{\mathbf{R} \in \mathbb{R}^{n \times n} \mid \mathbf{R}^T \mathbf{R} = \mathbf{I}, \det(\mathbf{R}) = 1 \right\},$$ is the group of rotation matrices, $\mathbf{x}$ is the position of the reference point $P$ in $\mathcal{F}$, $\mathbf{R}$ is the body-to-world rotation matrix relating $\mathcal{B}$ and $\mathcal{F}$, and $n \in \left\{2, 3\right\}$, depending on if the manipulation task is in the plane or in general 3D space, respectively. For simplicity, we say $$\mathbb{R}^n \times SO(n) \cong SE(n),$$ i.e., we denote this set as the Special Euclidean group, $$SE(n) = \left\{\left[\begin{array}{cc} \mathbf{R} & \mathbf{v}\\ \bm{0} & 1\end{array}\right] \mid \mathbf{R} \in SO(n), \mathbf{v} \in \mathbb{R}^n \right\}$$which is the set of $(n+1) \times (n+1)$ transformation matrices, since these spaces are homeomorphic. Abusing notation, we define $\dot{\mathbf{q}} = \left[\dot{\mathbf{x}}^T, \bm{\omega}^T\right]^T \in \mathbb{R}^{3n-3}$ which appends $\dot{\mathbf{x}}$, the velocity of $P$ in $\mathcal{F}$, with $\bm{\omega}$, the body's angular rate(s) in $\mathcal{F}$. We now define the control problem considered in this work.

 \begin{problem}
Assume each agent can measure the body's pose $\mathbf{q}$ and rates $\dot{\mathbf{q}}$, but has no knowledge of its mass properties ($m$, $\mathbf{I}_{cm}$) or geometry ($\mathbf{r}_p$, $\mathbf{r}_i$). Design a decentralized control law for each agent which allows the body's pose $\mathbf{q}$ to track a desired trajectory $\mathbf{q}_d(t) \in SE(n)$ asymptotically.
\label{pb:control}
\end{problem}

It is apparent that solving Problem \ref{pb:control} achieves collaborative manipulation under relatively few assumptions on the information available to each agent. However, the assumption of a common measurement $(\mathbf{q},\dot{\mathbf{q}})$ warrants additional motivation.

We aim to collaboratively control both the linear and rotational dynamics of the object; but in defining a reference trajectory $\mathbf{q}_d(t) \in SE(n)$, we must additionally define which point on the body we wish to track this trajectory. Yet, since we assume the body's geometry is unknown, this problem is circular: even if the agents can measure their own states, they have no way of computing or estimating the state of the measurement point, making the problem ill-posed.

% In this work, since we aim to track general trajectories on $SE(n)$, we allow the agents to simply measure the pose of the reference point; this can be achieved in practice by placing a sensor anywhere on the body, or by having a single agent broadcast its measurements to the other agents. We note this broadcast communication architecture is much simpler than a point-to-point or ``mesh'' network.

In this work, we chose to resolve this ambiguity by allowing the agents to simply measure the pose and velocity of the reference point. In practice, this could be achieved by one agent broadcasting its measurements to the group; this would require a very simple ``broadcast'' network architecture that could easily scale by adding more robots equipped to receive the measurement. Alternatively, we note that the average of the robots' individual linear and angular velocities is equal to the body's linear and angular rates at the robots' centroid. Thus, the robots could use either an `all-to-all'' communication scheme, or a mesh network and linear consensus \cite{Olfati-Saber:2004aa} to compute the averaged measurement.

However, different assumptions can be made if we still wish to solve Problem \ref{pb:control}, but without the use of a common measurement. One option is to restrict the class of desired trajectories $\mathbf{q}_d$; if the body's angular velocity $\bm{\omega} = \bm{0}$, then the velocity is uniform across all points on the body, and robots can use their own velocity measurements $\dot{\mathbf{x}}_i$ in order to track a desired linear velocity $\dot{\mathbf{x}}_d$. Similarly, if we only seek to control the body's orientation $\mathbf{R}$, if the agents share a reference frame, then their pose measurements are identical, making the problem again well-posed. 

Finally, we can eliminate the need for a common measurement if we assume the agents know their position $\mathbf{r}_i$ with respect to the measurement point, as in \cite{Liu:1998aa, Verginis:2017aa}. Using this information, the agents can easily compute the position and velocity of the reference point from their own measurements. However, we argue such an assumption is quite restrictive, since each agent must now have some way of localizing itself on the body before starting a manipulation task. We further note that such a method is not robust to errors in measuring $\mathbf{r}_i$, which could introduce unmodeled sources of uncertainty.

To solve Problem \ref{pb:control}, we propose a modified form of the Slotine-Li adaptive manipulator controller \cite{Slotine:1987aa}, which is generalized to the multi-agent case, and can adapt to the unknown mass and geometric parameters. We achieve tracking of a desired pose trajectory $\mathbf{q}_d(t) \in SE(3)$ by introducing a composite error on $SE(3)$, and show that the proposed controller indeed results in asymptotic tracking, while also ensuring all signals in the system remain bounded.
\section{Preliminaries}
\subsection{$SO(3)$ and Other Mathematical Preliminaries}
In this work, we are interested in controlling both the position and orientation of a rigid body, most generally in 3D space. While there exist numerous ways to parameterize the orientation of a rigid body, it is well-known that three-parameter representations such as Euler angles are singular, meaning they experience configurations in which the body's orientation and rotational dynamics are not uniquely defined. In this work we instead represent orientation using a full rotation matrix $\mathbf{R} \in SO(3)$, which is both unique and non-singular. 

Since $SO(3)$ is a Lie group, it admits a Lie algebra which captures the tangent space of the group at its identity element; this Lie algebra is in fact the set of skew-symmetric matrices $\mathfrak{so}(3)$, where $$\mathfrak{so}(3) = \left\{ \mathbf{A} \in \mathbb{R}^{3 \times 3} \mid \mathbf{A}^T = -\mathbf{A} \right\}.$$

We further note that $\mathfrak{so}(3)$ forms a three-dimensional vector space: any matrix in $\mathfrak{so}(3)$ can be mapped to a vector in $\mathbb{R}^3$ and vice versa. Let us denote by $(\cdot)^\times$ the map from $\mathbb{R}^3$ to $\mathfrak{so}(3),$ i.e., for $\mathbf{v} = \left[v_1, v_2, v_3\right]^T \in \mathbb{R}^3,$
\begin{equation*}
    \mathbf{v}^\times = \left[\begin{array}{ccc} 0 &-v_3 &v_2\\
    v_3 &0 &-v_1\\
    -v_2 &v_1 &0 \end{array}\right].
\end{equation*}
We further denote the inverse map as $(\cdot)^\vee$.

Let us define the operator $\mathbb{P}_a : \mathbb{R}^{n \times n} \mapsto \mathbb{R}^{n \times n},$ where, for  $\mathbf{A} \in \mathbb{R}^{n\times  n},$ $$\mathbb{P}_a(\mathbf{A}) = \textstyle \frac{1}{2} \left(\mathbf{A} - \mathbf{A}^T\right),$$ which we call the ``skew symmetric'' part of $\mathbf{A}$. Further, let us define $\mathbb{P}_s : \mathbb{R}^{n\times n} \mapsto \mathbb{R}^{n \times n},$ where $$\mathbb{P}_s(\mathbf{A}) = \textstyle \frac{1}{2} \left(\mathbf{A} + \mathbf{A}^T \right),$$ is the ``symmetric'' part. It is easy to verify that $\mathbf{A} = \mathbb{P}_a(\mathbf{A}) + \mathbb{P}_s(\mathbf{A})$, and for $n = 3,$ $\mathbb{P}_a(\mathbf{A}) \in \mathfrak{so}(3)$.

We further define the usual inner product on $\mathbb{R}^{n \times n}$, $$\left\langle \mathbf{A}, \mathbf{B} \right\rangle = \tr(\mathbf{A}^T \mathbf{B}).$$
Beyond the usual properties of an inner product, it is easy to verify that $\left\langle \mathbb{P}_a(\mathbf{A}), \mathbb{P}_s(\mathbf{A})\right\rangle = 0,$ and $\left\langle \mathbb{P}_a (\mathbf{A}), \mathbb{P}_a(\mathbf{A})\right\rangle = \frac{1}{2}\lvert\lvert \mathbb{P}_a(\mathbf{A})^\vee \rvert \rvert^2.$

Finally, we say a symmetric matrix $\mathbf{A} = \mathbf{A}^T \in \mathbb{R}^{n\times n}$ is positive semidefinite if, for all $\mathbf{x} \in \mathbb{R}^{n},$ $\mathbf{x}^T \mathbf{A} \mathbf{x} \geq 0.$ We say $\mathbf{A}$ is positive definite if the inequality holds strictly for all nonzero $\mathbf{x}$.

\subsection{Convergence of Nonlinear Systems}
% \textcolor{blue}{PC: need to change this discussion to ``almost global convergence'' instead.}

Due its topology, no dynamical system on $SE(3)$ has a globally asymptotically stable equilibrium \cite{Bhat:1998aa}. Thus, we turn to relaxed notions of stability and convergence for systems with rotational dynamics. Here we state a slightly relaxed notion of convergence, \textit{almost global  convergence}.

\begin{definition}[Almost Global Convergence]
Let $\mathbf{x}_e$ be an equilibrium of a dynamical system $\dot{\mathbf{x}} = \mathbf{f}(\mathbf{x}),$ where $\mathbf{x} \in \mathcal{X}$. We say trajectories of the system almost globally converge to $\mathbf{x}_e$ if, for some open, dense subset of the state space $\mathcal{A} \subset \mathcal{X}$, with closure $\bar{\mathcal{A}} = \mathcal{X}$, all trajectories with $\mathbf{x}(0) \in \mathcal{A}$ tend towards $\mathbf{x}_e$ asymptotically.

Similarly, we say trajectories of the system almost globally exponentially converge to $\mathbf{x}_e$ if all trajectories with $\mathbf{x}(0) \in \mathcal{A}$ converge exponentially to $\mathbf{x}_e$.
\end{definition}

\section{Equations of Motion for Collaborative~Manipulation}
In this section, we will formulate the equations of motion for collaborative manipulation tasks, both in the plane and in three dimensions, and show they can be represented using the Lagrangian equations for robot motion \cite{Murray:2017aa}, commonly termed ``the manipulator equations.''  
\subsection{Rigid Body Dynamics in Three Dimensions}
Let us consider the dynamics of a rigid body being manipulated as shown in Figure \ref{fig:schem}.

The body's linear motion is given by Newton's law,
\begin{equation}
    \mathbf{F} = m\mathbf{a}_{cm},
    \label{eq:newton}
\end{equation}
where $\mathbf{F}$ is the total force applied to the body, and $\mathbf{a}_{cm}$ is the acceleration of the body's center of mass in $\mathcal{F}$. We seek to relate $\mathbf{a}_{cm}$ to $\mathbf{a}_{p}$, the acceleration of $P$ in $\mathcal{F}$. We first relate the points' velocities,
\begin{equation*}
    \mathbf{v}_p = \mathbf{v}_{cm} + \bm{\omega} \times \mathbf{R} \mathbf{r}_p,
\end{equation*}
where $\bm{\omega}$ is the angular velocity of $\mathcal{B}$ in $\mathcal{F}$, and $\mathbf{R}$ is the rotation matrix from $\mathcal{B}$ to $\mathcal{F}$. This can be differentiated to yield
\begin{equation}
    \mathbf{a}_p = \mathbf{a}_{cm} + \bm{\alpha} \times \mathbf{R} \mathbf{r}_p + \bm{\omega} \times (\bm{\omega} \times \mathbf{R} \mathbf{r}_p),
    \label{eq:accel}
\end{equation}
where $\bm{\alpha}$ is $\mathcal{B}$'s angular acceleration in $\mathcal{F}$.
We can substitute \eqref{eq:accel} into \eqref{eq:newton} to yield
\begin{equation}
    \mathbf{F} = m\mathbf{a}_p - m(\bm{\alpha} \times \mathbf{R}\mathbf{r}_p) - m \bm{\omega} \times (\bm{\omega} \times \mathbf{R} \mathbf{r}_p).
    \label{eq:linear}
\end{equation}

The body's rotational dynamics are given by Euler's rigid body equation,
\begin{equation}
    \mathbf{M}_p = \mathbf{R} \mathbf{J}_p \mathbf{R}^T \bm{\alpha} + \bm{\omega} \times (\mathbf{R}\mathbf{J}_p\mathbf{R}^T \bm{\omega}) - m \mathbf{R} \mathbf{r}_p \times \mathbf{a}_{p},
    \label{eq:rotation}
\end{equation}
where $\mathbf{J}_p$ is the body's inertia matrix about $\mathbf{b}_p$, which is given by
\begin{equation*}
    \mathbf{J}_p = \mathbf{J}_{cm} + m\left((\mathbf{r}_p^T \mathbf{r}_p) \mathbf{I} - \mathbf{r}_p \mathbf{r}_p^T\right),
\end{equation*}
where $\mathbf{I}$ denotes the identity matrix. 

We again define the configuration variable $\mathbf{q} = \left[\mathbf{x}, \mathbf{R}\right] \in \mathbb{R}^3 \times SO(3),$ where $\mathbf{x}$ is the position of the measurement point in $\mathcal{F}$. Abusing notation, we can let $\dot{\mathbf{q}}, \ddot{\mathbf{q}} \in \mathbb{R}^6$, where $\dot{\mathbf{q}} = \left[\dot{\mathbf{x}}^T,\bm{\omega}^T\right]^T $ represents the body's linear and angular rates, and $\ddot{\mathbf{q}} = \left[\ddot{\mathbf{x}}^T, \bm{\alpha}^T\right]^T$ represents its linear and angular accelerations. 

Using these configuration variables, we can now write \eqref{eq:accel} and \eqref{eq:rotation} more compactly as 
\begin{equation}
    \bm{\tau} = \mathbf{H}(\mathbf{q})\ddot{\mathbf{q}} + \mathbf{C}(\mathbf{q},\dot{\mathbf{q}})\dot{\mathbf{q}},
    \label{eq:manip_eq}
\end{equation}
where $\bm{\tau}$ is the total wrench applied to the body about the measurement point, where
\begin{equation}
     \mathbf{H}(\mathbf{q}) = \left[\begin{array}{cc} m\mathbf{I} & m(\mathbf{R} \mathbf{r}_p)^\times \\ -m(\mathbf{R} \mathbf{r}_p)^\times & \mathbf{R} \mathbf{J}_p \mathbf{R}^T \end{array} \right]
    \label{eq:H_SO(3)}
\end{equation}
is the system's inertia matrix (which is symmetric, positive definite), and $$
\resizebox{\columnwidth}{!}{%
    $
        \mathbf{C}(\mathbf{q},\dot{\mathbf{q}}) = \left[\begin{array}{cc} \bm{0}_3  & m\bm{\omega}^\times (\mathbf{R} \mathbf{r}_p)^\times \\ -m\bm{\omega}^\times (\mathbf{R} \mathbf{r}_p)^\times & \bm{\omega}^\times \mathbf{R}\mathbf{I}_p\mathbf{R}^T -m((\mathbf{R}\mathbf{r}_p)^\times \dot{\mathbf{x}})^\times \end{array} \right]
     $%
    }$$
is a matrix which contains centrifugal and Coriolis terms, with the $3\times3$ zero matrix denoted by $\bm{0}_3$. 

Further, we can show that these matrices have a number of interesting properties. Specifically, $\mathbf{H}(\mathbf{q})$ is positive definite for all $\mathbf{q}$, and $\dot{\mathbf{H}} - 2\mathbf{C}$ is skew-symmetric. These properties can be viewed as matrix versions of common properties of Hamiltonian systems (e.g. conservation of energy). Proof of these properties is included in Appendix \ref{sec:matrix_properties}. 

One can easily derive similar expressions for the two-dimensional case, when the body is restricted to move in the plane; we omit these expressions for brevity.

\subsection{Grasp Matrix and its Properties}
While the dynamics in \eqref{eq:manip_eq} describe the body's motion under the total wrench $\bm{\tau}$ about $\mathbf{b}_p,$ we must further express $\bm{\tau}$ as a function of each robot's applied wrench $\bm{\tau}_i$. We can write $\bm{\tau}$ as 
\begin{equation}
    \bm{\tau} = \sum_{i=1}^{N} \mathbf{M}(\mathbf{q},
    \mathbf{r}_i) \bm{\tau}_i,
\end{equation}
where $\mathbf{M}$ denotes the grasp matrix, which is given by
\begin{equation*}
\mathbf{M}(\mathbf{q},\mathbf{r}_i) = \left[\begin{array}{cc} \mathbf{I} & \bm{0}_3 \cr 
(\mathbf{R} \mathbf{r}_i)^\times & \mathbf{I} \end{array} \right],
\end{equation*}
where agent $i$ grasps the object at some point $\mathbf{r}_i$ from the measurement point. We denote the set of such grasp matrices as $\mathcal{M}$.

The grasp matrix has some unique properties, specifically,
\begin{equation*}
    \mathbf{M}(\mathbf{r}_i)^{-1} = \left[\begin{array}{cc} \mathbf{I} & \bm{0}_3 \\ 
-(\mathbf{R} \mathbf{r}_i)^\times & \mathbf{I} \end{array} \right] = \mathbf{M}(-\mathbf{r}_i),
\end{equation*}
and 
\begin{align}
    \mathbf{M}(\mathbf{r}_i) \mathbf{M}(\mathbf{r}_j) &= \left[\begin{array}{cc} \mathbf{I} & \bm{0}_3 \cr 
\left(\mathbf{R}(\mathbf{r}_i+\mathbf{r}_j)\right)^\times & \mathbf{I} \end{array} \right] \nonumber\\ &= \mathbf{M}(\mathbf{r}_i) + \mathbf{M}(\mathbf{r}_j) - \mathbf{I}.  \label{eq:grasp_mats}
\end{align}
Thus, interestingly, for two matrices in $\mathcal{M}$, their product is simply equal to their sum, minus the identity.

\section{Decentralized Adaptive Trajectory Tracking}
We will now propose a decentralized adaptive controller suitable for accomplishing the collaborative manipulation tasks outlined previously, and provide proofs of boundedness of signals, and almost global asymptotic tracking.

\subsection{Proposed Composite Error on $SE(3) \times \mathbb{R}^6$}
We begin by proposing a composite error for 3D poses and their derivatives, which lie in $SE(3)\times  \mathbb{R}^6$. Adaptive controllers for manipulators in literature (cf., \cite[Chapter~9]{Slotine:1991aa}) typically use a composite error $\mathbf{s}$ which linearly combines position and velocity errors, and exhibits linear, exponentially stable dynamics on the surface $\mathbf{s} = 0.$ Such a variable is appropriate for configuration variables (such as positions and small joint angles) which lie in Cartesian space.

However, since we aim to track both a desired position $\mathbf{x}_d$ and orientation $\mathbf{R}_d$, the system's dynamics do not evolve naturally in Cartesian space. We thus propose a composite error $\mathbf{s}$ which is non-singular, inherits the topology of $SE(3)$, and exhibits almost global convergence when on the surface $\mathcal{D} = \left\{(\mathbf{q}_e,\dot{\mathbf{q}}_e) \in SE(3) \times \mathbb{R}^6 \mid \mathbf{s}(\mathbf{q}_e, \dot{\mathbf{q}}_e) = 0\right\}$. To achieve this, we adapt the sliding variable on $SO(3)$ from \cite{Gomez-Cortes:2019aa}, and combine it with a typical linear term on $\mathbb{R}^3$ to define a composite error on $SE(3)$. We then show almost all trajectories of the system converge to the desired trajectory when $\mathbf{s} = 0.$

% Let us define the position of the measurement point on the body as $\mathbf{x} \in \mathbb{R}^3$, and the body's orientation as $\mathbf{R} \in SO(3)$, where $SO(3)$ is the group of rotation matrices in $\mathbb{R}^3$, $$SO(3) := \left\{\mathbf{R} \in \mathbb{R}^{3\times3} \mid \mathbf{R}^T\mathbf{R} = \mathbf{I}, \det(\mathbf{R}) = 1\right\}.$$
%  We further denote the body's global-frame linear and angular velocities as $\dot{\mathbf{x}} \in \mathbb{R}^3,$ and $\bm{\omega} \in \mathbb{R}^3$, respectively, where the attitude dynamics satisfy the differential equation $$\dot{\mathbf{R}} = \bm{\omega}^\times \mathbf{R}.$$ \textcolor{red}{Mac: The previous things were already defined in problem set up section.} 

We aim to control the body to track a desired position $\mathbf{x}_d(t)$, and a desired orientation, $\mathbf{R}_d$, which satsifies $$\dot{\mathbf{R}}_d = \bm{\omega}_d^\times \mathbf{R}_d,$$ for some desired angular velocity $\bm{\omega}_d$. Let us define the rotational and angular velocity errors as $\mathbf{R}_e = \mathbf{R}_d^T \mathbf{R},$ and $\bm{\omega}_e = \bm{\omega}-\bm{\omega}_d,$ noting $$(\mathbf{R}_e, \bm{\omega}_e) = (\mathbf{I}, \bm{0}) \iff (\mathbf{R}, \bm{\omega}) = (\mathbf{R}_d, \bm{\omega}_d).$$

We can then obtain the rotation error dynamics,
\begin{align}
    \dot{\mathbf{R}}_e &= \dot{\mathbf{R}}_d^T \mathbf{R} + \mathbf{R}_d^T \dot{\mathbf{R}} = \mathbf{R}_d^T(\bm{\omega}-\bm{\omega}_d)^\times \mathbf{R}, \nonumber \\
    &= \left(\mathbf{R}_d^T(\bm{\omega}-\bm{\omega}_d)\right)^\times \mathbf{R}_d^T\mathbf{R} = \left(\mathbf{R}_d^T \bm{\omega}_e\right)^\times \mathbf{R}_e. {\label{eq:rot_err_dynamics}}
\end{align}

Finally, we propose a composite error $\mathbf{s} : SE(3) \times \mathbb{R}^6 \mapsto \mathbb{R}^6,$ which is given by 
\begin{equation*}
    \mathbf{s} = \left[\begin{array}{c} \mathbf{s}_\ell \\ \bm{\sigma} \end{array}\right] = \left[\begin{array}{c} \dot{\tilde{\mathbf{x}}} + \lambda \tilde{\mathbf{x}}\\ \bm{\omega}_e + \lambda \mathbf{R}_d \mathbb{P}_a(\mathbf{R}_e)^\vee \end{array}\right],
\end{equation*}
where $\widetilde{\mathbf{x}} = \mathbf{x} - \mathbf{x}_d$ is the linear tracking error, and $\mathbf{s}_\ell : \mathbb{R}^3 \times \mathbb{R}^3 \mapsto \mathbb{R}^3$ and $\bm{\sigma} : SO(3) \times \mathbb{R}^3 \mapsto \mathbb{R}^3$ are composite errors for the object's linear and angular dynamics, respectively.

\subsection{System Behavior on $\mathbf{s} = 0$}
\label{sec:composite_proof}
We now consider the system's behavior when constrained to the surface $\mathbf{s} = 0$. The behavior of the linear dynamics is obvious; imposing the condition $\mathbf{s}_\ell = 0$ results in the linear error dynamics $$\dot{\tilde{\mathbf{x}}} = -\lambda \tilde{\mathbf{x}},$$ which defines a stable linear system. Thus, if the system lies on the surface, the position error $\tilde{\mathbf{x}}$ tends exponentially to zero, with time constant $\frac{1}{\lambda}.$

Thus, we turn our attention to the rotational error dynamics defined by $\bm{\sigma} = 0$. Imposing this condition results in the rotational dynamics
\begin{align*}
    \dot{\mathbf{R}}_e &= \left(-\lambda \mathbf{R}_d^T \mathbf{R}_d \mathbb{P}_a(\mathbf{R}_e)^\vee\right)^\times \mathbf{R}_e,\\
    &= -\lambda \mathbb{P}_a(\mathbf{R}_e)\mathbf{R}_e,
\end{align*}
using the rotation error dynamics given by \eqref{eq:rot_err_dynamics}. We note that $(\mathbf{R}_e, \bm{\omega}_e) = (\mathbf{I},\bm{0})$ both lies on the surface $\mathbf{s} = 0$, and is an equilibrium of the reduced-order rotational dynamics.

\begin{theorem}
% \textcolor{blue}{PC: We may want to reason here about stability of the equilibrium, since the reduced-order system is autonomous.}
Almost all trajectories $(\mathbf{R}_e, \bm{\omega}_e)$ of the reduced-order system given by $\bm{\sigma} = 0$ converge exponentially to the desired equilibrium $(\mathbf{I}, \bm{0}).$
\end{theorem}
\begin{proof}
Consider the Lyapunov-like function $$V_R(t) = \tr\left(\mathbf{I}-\mathbf{R}_e\right) > 0,$$ which is lower-bounded by zero since $-1 \leq \tr(\mathbf{R}_e) \leq 3$ for all $\mathbf{R}_e \in SO(3)$, with $\tr(\mathbf{R}_e) = 3$ iff $\mathbf{R}_e = \mathbf{I}$. 
The time derivative of $V_R$ is given by
\begin{align*}
    \dot{V}_R &= -\tr(\dot{\mathbf{R}}_e) = \lambda \left\langle -\mathbb{P}_a(\mathbf{R}_e), 
    \mathbf{R}_e \right\rangle,\\
    &= -\lambda \left\langle \mathbb{P}_a(\mathbf{R}_e), \mathbb{P}_a(\mathbf{R}_e) + \mathbb{P}_s(\mathbf{R}_e)\right\rangle,\\
    &= \textstyle-\frac{\lambda}{2} \lvert\lvert \mathbb{P}_a(\mathbf{R}_e)^\vee \rvert\rvert^2 \leq 0,
\end{align*}
using the matrix inner product $\left\langle\mathbf{A},\mathbf{B}\right\rangle = \tr(\mathbf{A}^T\mathbf{B}),$ and the fact that $\left\langle \mathbb{P}_a(\mathbf{A}), \mathbb{P}_s(\mathbf{A})\right\rangle = 0,$ for $\mathbf{A}, \mathbf{B} \in \mathbb{R}^{3\times3}.$ 

In order to show almost global exponential convergence to $(\mathbf{I},\bm{0})$, we seek a strictly positive constant $\alpha > 0$ such that $$\dot{V}_R \leq -\alpha V_R.$$
To find this constant, let us recall that for the rotation matrix $\mathbf{R}_e$ there is an equivalent unit quaternion representation, $$\mathbf{z}(t) = q_o + q_1 i + q_2 j + q_3 k,$$ with $\tr(\mathbf{R}_e) = 4 q_0^2 -1,$ and $q_o^2 + q_1^2 + q_2^2 + q_3^2 = 1.$ Further, we can rewrite $\dot{V}_R,$
\begin{align*}
    \dot{V}_R &= \frac{\lambda}{2} \tr\left(\mathbf{R}_e \mathbf{R}_e - \mathbf{I} \right) = 2\lambda \left((q_0^2 - q_1^2 - q_2^2 - q_3^2)^2 - 1\right),
\end{align*} 
where the second equality follows from quaternion multiplication. Since $\mathbf{z}$ has unit norm, we can write
\begin{align*}
    \dot{V}_R &= 2 \lambda \Big(\left(2q_0^2 - 1\right)^2 - 1\Big) = 8 \lambda q_0^2 \left(q_0^2 - 1\right),\\
    &= - \left(2 \lambda q_0^2\right) V_R.
\end{align*}
Further, we note $q_0^2 = 0$ if and only if $\tr(\mathbf{R}_e) = -1$, which corresponds to the global maximum of $V_R.$

Let us assume our trajectory does not start at this point, i.e., there exists some $\epsilon > 0$ such that $4~\geq~V_R(0) + \epsilon.$ Then, since $\dot{V}_R \leq 0,$
\begin{align*}
    q_0^2(t) = 1 - \textstyle \frac{1}{4}V_R(t) \geq 1 - \frac{1}{4} V_R(0) > \frac{\epsilon}{4} > 0.
\end{align*}
Thus, for all $t \geq 0,$ we have $$\dot{V}_R \leq - \alpha V_R,$$ for $\alpha = \frac{1}{2}\lambda \epsilon > 0.$ Thus, almost all trajectories on $\mathbf{\sigma = 0}$ converge exponentially to the desired equilibrium.

\end{proof}

Thus, we can conclude that almost all trajectories satisfying $\mathbf{s} = 0$ will exponentially converge to tracking the desired position $\mathbf{x}_d$ and orientation $\mathbf{R}_d$\footnote{Note the composite error $\mathbf{s}$ discussed here is closely related to the ``sliding variable'' commonly used in sliding mode control \cite[Ch. 7]{Slotine:1991aa}, since, when when we constrain $\mathbf{s}=0$, the reduced-order system dynamics converge exponentially to tracking. However, the controller \eqref{eq:ctrl_1} proposed in the following discussion is not, in fact, a sliding mode controller because it does not include the discontinuous term typically used to guarantee finite-time convergence to $\mathbf{s}=0$; we show the proposed controller instead exhibits asymptotic tracking.}.

We further note that $\mathbf{s}$ can be interpreted as a ``velocity error,'' since we can write $$\mathbf{s} = \dot{\mathbf{q}} - \dot{\mathbf{q}}_r,$$ where $\dot{\mathbf{q}}_r$ is a ``reference velocity'' which augments the desired velocity $\dot{\mathbf{q}}_d$ with an additional term based on the pose error, $$\dot{\mathbf{q}}_r = \left[\begin{array}{c} \dot{\mathbf{x}}_d - \lambda \widetilde{\mathbf{x}} \\ \bm{\omega}_d - \lambda \mathbf{R}_d \mathbb{P}_a(\mathbf{R}_e)^\vee \end{array} \right].$$

\subsection{Decentralized Adaptive Trajectory Control}
We now return to Problem \ref{pb:control}, decentralized adaptive control of a common payload. We consider the Lagrangian dynamics
\begin{equation}
    \mathbf{H}(\mathbf{q})\ddot{\mathbf{q}} + \mathbf{C}(\mathbf{q},\dot{\mathbf{q}}) + \mathbf{g} = \textstyle \sum_{i=1}^{N} \mathbf{M}_i(\mathbf{q})\bm{\tau}_i,
    \label{eq:ham_dynamics}
\end{equation}
which, as shown previously, hold for collaborative manipulation tasks on both $SE(2)$ and $SE(3)$. We note that while the manipulation dynamics do not include a term $\mathbf{g},$ which models conservative forces such as gravity, we include it here for completeness, and to conform to the traditional ``manipulator equation.''

Let us consider a Lyapunov-like function candidate
\begin{align}
\begin{split}
    \textstyle V(t) = \frac{1}{2}\big[ \mathbf{s}^T \mathbf{H} \mathbf{s} + \sum_{i=1}^{N} &\widetilde{\mathbf{o}}_i^{\,T} \bm{\Gamma}_o^{-1}\widetilde{\mathbf{o}}_i + \widetilde{\mathbf{r}}_i^{\,T} \bm{\Gamma}_r^{-1} \widetilde{\mathbf{r}}_i \big],
    \label{eq:lyap}
\end{split}
\end{align}
where $\widetilde{\mathbf{o}}_i = \hat{\mathbf{o}}_i - \mathbf{o}_i$ is a parameter estimation error for agent $i$, with $\mathbf{o}_i$ being a vector of physical constants that define the payload's dynamics, and $\hat{\mathbf{o}}_i$ being agent $i$'s estimate of these parameters. Similarly, we let $\widetilde{\mathbf{r}}_i$ be the difference between the vector $\mathbf{r}_i$ from the measurement point to agent $i$, and the agent's estimate $\hat{\mathbf{r}}_i$. We further let $\bm{\Gamma}_o$ and $\bm{\Gamma}_r$ be positive definite matrices which correspond to adaptation gains.

Differentiating $V$ yields
\begin{equation*}
     \textstyle \dot{V}(t) = \mathbf{s}^T \mathbf{H} \dot{\mathbf{s}} + \frac{1}{2} \mathbf{s}^T \dot{\mathbf{H}} \mathbf{s} + \sum_{i=1}^{N} \widetilde{\mathbf{o}}_i^{\,T} \bm{\Gamma}_o^{-1} \dot{\hat{\mathbf{o}}} + \widetilde{\mathbf{r}}_i^{\,T} \bm{\Gamma}_r^{-1}\dot{\hat{\mathbf{r}}}_i,
\end{equation*}
using the fact that the parameter error derivatives $\dot{\widetilde{\mathbf{o}}}_i, \dot{\widetilde{\mathbf{r}}}_i$ are equal to the parameter estimate derivatives $\dot{\hat{\mathbf{o}}}_i, \dot{\hat{\mathbf{r}}}_i$, since the true parameter values are constant.

Further, since we can write $\dot{\mathbf{s}} = \ddot{\mathbf{q}} - \ddot{\mathbf{q}}_r$, we can expand the term
\begin{align*}
      \textstyle \mathbf{s}^T \mathbf{H} \dot{\mathbf{s}} &= \mathbf{s}^T\mathbf{H}(\ddot{\mathbf{q}} - \ddot{\mathbf{q}}_r),\\
     &= \mathbf{s}^T \left(\textstyle \sum_{i=1}^{N} \mathbf{M}_i \bm{\tau}_i - \mathbf{C} \dot{\mathbf{q}} - \mathbf{g} - \mathbf{H}\ddot{\mathbf{q}}_r\right),\\
     &= \mathbf{s}^T\left(\textstyle\sum_{i=1}^{N} \mathbf{M}_i\bm{\tau}_i - \mathbf{H}\ddot{\mathbf{q}}_r - \mathbf{C}\dot{\mathbf{q}}_r - \mathbf{g} - \mathbf{C} \mathbf{s}\right),
\end{align*}
using the system dynamics \eqref{eq:ham_dynamics} to substitute for $\mathbf{H}\ddot{\mathbf{q}}$, and the fact that $\dot{\mathbf{q}} = \mathbf{s} + \dot{\mathbf{q}}_r$.

Further, since the system is over-actuated, meaning it has many more control inputs than degrees of freedom, there exist many sets of control inputs $\bm{\tau}_i$ which produce the same object dynamics. Thus, it is sufficient for each agent to take some portion of the control effort, which, combined, enables tracking of the desired trajectory. To this end, let us introduce a set of $N$ positive constants $\mathbf{\alpha}_i$, with $\sum_{i=1}^{N} \alpha_i = 1$; a natural choice is $\alpha_i = \frac{1}{N},$ which we use throughout the paper.

% Using these constants, we can allocate the total control effort needed for feedback linearization among the agents, writing $$\mathbf{M}_i \bm{\tau}_i = \alpha_i\left(\mathbf{H}\ddot{\mathbf{q}} + \mathbf{C}\dot{\mathbf{q}} + \mathbf{g}\right)$$
% for each $i$. It is important to note that these constants, in effect, only scale the physical parameters of the system, meaning they can be adapted to and do not need to be known \textit{a priori}. 

Using these constants, we can further exploit an important property of this system, \textit{linear parameterizability}. Importantly, for the control tasks considered in this work, the physical parameters to be estimated appear only linearly in the dynamics. Thus, we can define a known regressor matrix $\mathbf{Y}_o(\mathbf{q},\dot{\mathbf{q}},\dot{\mathbf{q}}_r, \ddot{\mathbf{q}}_r),$ such that $$\alpha_i \left(\mathbf{H}(\mathbf{q})\ddot{\mathbf{q}}_r + \mathbf{C}(\mathbf{q},\dot{\mathbf{q}})\dot{\mathbf{q}}_r + \mathbf{G} \right) = \mathbf{Y}_o(\mathbf{q},\dot{\mathbf{q}},\dot{\mathbf{q}}_r, \ddot{\mathbf{q}}_r) \mathbf{o}_i,$$ where $\mathbf{o}_i = \alpha_i \mathbf{o}$ is a vector of the object's physical parameters, $\mathbf{o},$ such as mass or moments of inertia, scaled by $\alpha_i.$ We further note that 
\begin{align*}
\textstyle\sum_{i=1}^{N} \mathbf{Y}_o \mathbf{o}_i &= \textstyle\sum_{i=1}^{N} \alpha_i\left(\mathbf{H}\ddot{\mathbf{q}}_r + \mathbf{C}\dot{\mathbf{q}}_r +\mathbf{g}\right)\\
&= \mathbf{H}\ddot{\mathbf{q}}_r + \mathbf{C}\dot{\mathbf{q}}_r +\mathbf{g}.
\end{align*}

In short, the constants $\alpha_i$ are a mathematical convenience used to allocate control effort among the agents. Since they now seek to estimate the scaled physical parameters $\mathbf{o}_i$, the constants need not be known \textit{a priori}. Recalling the analogy between our proposed controller and quorum sensing \cite{Russo:2010aa}, we can interpret these constants as akin to coupling weights between the agents.

Using this regressor, we can write \begin{align}
    \dot{V}(t) &= \textstyle \sum_{i=1}^{N} \mathbf{s}^T\left(\mathbf{M}_i \bm{\tau}_i - \mathbf{Y}_o \mathbf{o}_i\right) + \frac{1}{2}\mathbf{s}^T\left(\dot{\mathbf{H}}-2\mathbf{C}\right)\mathbf{s} \nonumber \\ &\quad\quad + \widetilde{\mathbf{o}}_i^{\,T} \bm{\Gamma}_o^{-1} \dot{\hat{\mathbf{o}}} + \widetilde{\mathbf{r}}_i^{\,T} \bm{\Gamma}_r^{-1}\dot{\hat{\mathbf{r}}}_i.  \label{eq:vdot_partial}
\end{align}
As noted previously, $\dot{\mathbf{H}} - 2\mathbf{C}$ is skew-symmetric, thus its quadratic form is equal to zero for any vector $\mathbf{s}$.

We will now outine the proposed control and adaptation laws. Let \begin{equation}
    \bm{\tau}_i = \widehat{\mathbf{M}}_i^{-1}\mathbf{F}_i,
    \label{eq:ctrl_1}
\end{equation} where $\mathbf{M}_i(\hat{\mathbf{r}}_i)^{-1}$ is the inverse of a grasp matrix in $\mathcal{M}$ with a moment arm of $\hat{\mathbf{r}}_i$, and
\begin{equation}
 \mathbf{F}_i = \mathbf{Y}_o \hat{\mathbf{o}}_i - \mathbf{K}_D \mathbf{s},
 \label{eq:ctrl_2}
\end{equation}
where $\mathbf{K}_D$ is some positive definite matrix. Intuitively, the term $\mathbf{F}_i$ combines a feedforward term using the estimated object dynamics with a simple PD term, $-\mathbf{K}_D \mathbf{s}$. This control law alone would lead to successful tracking if all agents were attached at the measurement point. Each agent then multiplies $\mathbf{F}_i$ by $\widehat{\mathbf{M}}_i^{-1}$ in an effort to compensate for the torque its applied force generates about the measurement point. 

We further note, using \eqref{eq:grasp_mats}, that the product $$\mathbf{M}_i \widehat{\mathbf{M}}_i^{-1}=\big(\mathbf{I}-\widetilde{\mathbf{M}}_i\big),$$ where $\widetilde{\mathbf{M}}_i = \widehat{\mathbf{M}}_i - \mathbf{M}_i.$ We can further observe that the product $\widetilde{\mathbf{M}}_i \mathbf{F}_i$ is linearly parameterizable, with $$-\widetilde{\mathbf{M}}_i\mathbf{F}_i = \mathbf{Y}_g(\mathbf{F}_i,\mathbf{q})\widetilde{\mathbf{r}}_i.$$

Using this expression, we propose the adaptation laws
\begin{align}
    \dot{\hat{\mathbf{o}}}_i &= -\Gamma_o \mathbf{Y}_o(\mathbf{q},\dot{\mathbf{q}},\dot{\mathbf{q}}_r,\ddot{\mathbf{q}}_r)^T \mathbf{s}, \label{eq:adapt_1}\\
    \dot{\hat{\mathbf{r}}}_i &= -\Gamma_r \mathbf{Y}_g(\mathbf{F}_i,\mathbf{q})^{T}\mathbf{s}.
    \label{eq:adapt_2}
\end{align}

Note that all terms in the proposed control and adaptation laws can be computed using only local information (using the desired trajectory, shared measurement, and local parameter estimates). We now reach the main theoretical result of this paper, i.e., the boundedness of all closed-loop signals, and almost-global asymptotic convergence of the system trajectory to the desired trajectory.
\begin{theorem}
Consider the control laws \eqref{eq:ctrl_1}-\eqref{eq:ctrl_2}, and the adaptation laws \eqref{eq:adapt_1}-\eqref{eq:adapt_2}. Under the proposed adaptive controller, all closed-loop signals remain bounded, and the system converges to the desired trajectory from almost all initial conditions. \label{thm:controller}
\end{theorem}
\begin{proof}
Substituting these control and adaptation laws into the previous expression \eqref{eq:vdot_partial} of $\dot{V}(t)$ yields 
\begin{align*}
     \dot{V}(t) &= \textstyle\sum_{i=1}^{N} \mathbf{s}^T \left( \mathbf{F}_i + \mathbf{Y}_g \widetilde{\mathbf{r}}_i - \mathbf{Y}_o \mathbf{o}_i \right) \\ &\quad\quad + \widetilde{\mathbf{o}}_i^{\,T} \bm{\Gamma}_o^{-1} \dot{\hat{\mathbf{o}}} + \widetilde{\mathbf{r}}_i^{\,T} \bm{\Gamma}_r^{-1}\dot{\hat{\mathbf{r}}}_i,\\
    &= \textstyle \sum_{i=1}^{N} -\mathbf{s}^T \mathbf{K}_D \mathbf{s} + \widetilde{\mathbf{o}}_i^{\,T}\left(\mathbf{Y}_o \widetilde{\mathbf{o}}_i + \bm{\Gamma}_o^{-1}\dot{\hat{\mathbf{o}}}_i\right) \\ &\quad\quad + \widetilde{\mathbf{r}}^{\,T}\left(\mathbf{Y}_g \widetilde{\mathbf{r}}_i + \bm{\Gamma}_r^{-1}\dot{\hat{\mathbf{r}}}_i\right),
\end{align*}
which, upon substituting the adaptation laws \eqref{eq:adapt_1}-\eqref{eq:adapt_2}, yields
\begin{equation}
    \dot{V}(t) = N\left(-\mathbf{s}^T\mathbf{K}_D\mathbf{s}\right) \leq 0.
    \label{eq:vdot_final}
\end{equation}
As shown in Section \ref{sec:composite_proof}, when $\mathbf{s} = 0,$ almost all trajectories converge exponenially to tracking the desired trajectory. Thus, if we have $\mathbf{s} \rightarrow \bm{0}$, this is sufficient to show the body tracks the desired trajectory asymptotically. To this end, we now show that $\dot{V}(t) \rightarrow 0$ as $t \rightarrow \infty$, since from \eqref{eq:vdot_final} it is apparent that this is sufficient for $\mathbf{s} \rightarrow \bm{0}$. 

To do this, we again invoke Barbalat's Lemma \cite[Chapter~4.5]{Slotine:1991aa}, which states that $\dot{V} \rightarrow 0$ if $V$ has a finite limit, and $\dot{V}$ is uniformly continuous. From \eqref{eq:vdot_final}, we see that $V$ has a finite limit; thus, we show $\ddot{V}$ is bounded.

From the expression of $\ddot{V} = -2N\left(\mathbf{s}^T\mathbf{K}_D\dot{\mathbf{s}}\right)$, we have $\ddot{V}$ is bounded if $\mathbf{s}, \dot{\mathbf{s}}$ are bounded. Since $\dot{V} \leq 0$, and $V$ is lower-bounded, we have that $V$ is bounded, which implies $\mathbf{s}, \widetilde{\mathbf{o}}_i, \widetilde{\mathbf{r}}_i$ are bounded for all $i$. 

% \textcolor{red}{Mac: Barbalat's requires $V(t)$ approaches a limit and $\dot{V}(t)$ uniformly continuous to conclude $\dot{V}(t)$ approaches $0$.  Usually these conditions are satisfied by (i) $V(t)$ lower bounded, (ii) $\dot{V}(t)\le 0$, and (iii) $\ddot{V}(t)$ bounded.  Here (i) and (ii) imply $V(t)$ approaches a limit, and (iii) is sufficient to conclude that $\dot{V}(t)$ is uniformly continuous.  Thus (i), (ii), (iii) act as a corollary of Barbalat's to imply $\dot{V}(t)$ approaches $0$.  I know it's complicated, but we should get all the details right!}   

Further, we can write
\begin{align*}
    \mathbf{H}\dot{\mathbf{s}} &= \mathbf{H}(\ddot{\mathbf{q}}-\ddot{\mathbf{q}}_r),\\
    &= \textstyle \sum_{i=1}^{N} \left(\mathbf{F}_i + \mathbf{Y}_g \widetilde{\mathbf{r}}_i\right) - \mathbf{Y}_o \mathbf{o}_i - \mathbf{C}\mathbf{s},\\
    &= \textstyle \sum_{i=1}^{N} \left(\mathbf{Y}_o \widetilde{\mathbf{o}}_i + \mathbf{Y}_g \widetilde{\mathbf{r}}_i\right) - \left(N \mathbf{K}_D + \mathbf{C}\right) \mathbf{s},
\end{align*}
noting that all signals on the right-hand side are bounded (since the desired trajectory $\mathbf{q}_d$ and its derivatives are bounded by design). Finally, since $\mathbf{H}$ is positive definite, we know $\mathbf{H}^{-1}$ exists, which implies that $\dot{\mathbf{s}}$ is bounded. 

Thus, we have shown all signals are bounded. Further, since $\mathbf{s} \rightarrow \bm{0}$ as $t \rightarrow \infty$, then, except for the singularity $\tr\left(\mathbf{R}_e(0)\right) \neq -1,$ we have $\dot{\widetilde{\mathbf{q}}}, \widetilde{\mathbf{x}} \rightarrow \bm{0}, \mathbf{R}_d^T \mathbf{R} \rightarrow \mathbf{I}$ in the limit, which completes the proof. 
\end{proof}

Similarly to \cite{Slotine:1987aa}, Theorem \ref{thm:controller} shows boundedness of all closed-loop signals and almost-global asymptotic convergence to the desired trajectory.
Many variations and extensions of \cite{Slotine:1987aa} have been proposed.
For instance, \cite{Spong:1990aa} demonstrated that an additional term could be added to the Lyapunov-like function of \cite{Slotine:1987aa} to show that the desired trajectory and true parameters are stable in the sense of Lyapunov. However, this result relied on the particular linear structure of the composite error $\mathbf{s}$ which does not hold in our case.
%attempting to prove a similar stability result remains an interesting direction for future work.}

\subsection{Lack of Persistent Excitation}
\label{subsec:pe}
While we have shown that under the proposed controller, the tracking errors approach zero in the limit, we have not reasoned about the asymptotic behavior of the parameter errors $\widetilde{\mathbf{o}}_i, \widetilde{\mathbf{r}}_i$. While these signals must remain bounded, it is unclear under what conditions the parameter estimates converge to their true values.

For most (centralized) adaptive controllers, a condition termed ``persistent excitation'' is sufficient to show the parameter errors are driven to zero asymptotically. 
\begin{definition}[Persistent Excitation] Consider $\mathbf{W}(t) \in \mathbb{R}^{m \times n},$ a matrix-valued, time-varying signal. We say $\mathbf{W}(t)$ is persistently exciting if, for all $t > 0,$ there exist $\gamma, T > 0$ such that $$\frac{1}{T}\int_t^{t+T} \mathbf{W}(\tau)^T \mathbf{W}(\tau) d\tau > \gamma \mathbf{I}.$$ 
\end{definition}

Let us first consider the case of single-agent manipulation, i.e., the case where $N = 1.$ Further, let us define $\mathbf{Y}_1 = \left[ \mathbf{Y}_o, \mathbf{Y}_r \right],$ and $\widetilde{\mathbf{a}}_1 = \left[ \widetilde{\mathbf{o}}^{\,T}, \widetilde{\mathbf{r}}^{\,T} \right]^T ,$ meaning we stack the parameter errors row-wise, and their respective regressors column-wise. 

In this case, we can again write the closed-loop dynamics as $$\mathbf{H}\dot{\mathbf{s}} + \left(\mathbf{C}+\mathbf{K}_D\right)\mathbf{s} = \mathbf{Y}_1\widetilde{\mathbf{a}}_1,$$ which, since we know $\mathbf{s},\dot{\mathbf{s}} \rightarrow \bm{0}$ in the limit, implies that $\widetilde{\mathbf{a}}_1$ converges to the nullspace $\mathcal{N}(\mathbf{Y}_1)$ of $\mathbf{Y}_1;$ this does not, however, imply that $\widetilde{\mathbf{a}}_1 \rightarrow \bm{0}.$

Instead, it is sufficient 
to require that the matrix $\mathbf{Y}_{1,d}(t)$ is persistently exciting \cite{Slotine:1987ab}, where $\mathbf{Y}_{1,d}(t) = \mathbf{Y}_1(\mathbf{q}_d,\dot{\mathbf{q}}_d,\dot{\mathbf{q}}_d,\ddot{\mathbf{q}}_d)$. However, while there exist simple conditions on the reference signal to guarantee persistent excitation for linear, time-invariant systems, to date, such simple conditions do not exist for nonlinear systems. In practice, however, reference signals with a large amount of frequency content are typically used to achieve persistent excitation.

We now turn to the multi-agent case. Let us define $\mathbf{Y}_N = \left[ \mathbf{Y}_o, \mathbf{Y}_r, \cdots, \mathbf{Y}_o, \mathbf{Y}_r\right]$ and $\widetilde{\mathbf{a}}_N = \left[\widetilde{\mathbf{o}}_1^{\,T},\widetilde{\mathbf{r}}_1^{\,T}, \cdots, \widetilde{\mathbf{o}}_N^{\,T},\widetilde{\mathbf{r}}_N^{\,T}\right]^T,$ i.e., we stack $\mathcal{F}$ sets of the regressors $\mathbf{Y}_o, \mathbf{Y}_r$ column-wise, and append all $\mathcal{F}$ sets of parameter errors row-wise. This allows us to again write the closed-loop dynamics as $$\mathbf{H}\dot{\mathbf{s}} + \left(\mathbf{C}+\mathbf{K}_D\right)\mathbf{s} = \mathbf{Y}_N\widetilde{\mathbf{a}}_N.$$ Thus, to ensure that all parameter errors $\widetilde{\mathbf{a}}_i$ converge to zero, we hope to find conditions under which $\mathbf{Y}_{N,d}(\mathbf{q}_d,\dot{\mathbf{q}}_d,\dot{\mathbf{q}}_d,\ddot{\mathbf{q}}_d)$ is persistently exciting. Thus, for all $t > 0,$ we require that there exist $T, \gamma > 0$ such that $$\int_{t}^{t+T} \mathbf{Y}_{N,d}^T(\tau) \mathbf{Y}_{N,d}(\tau) d\tau \geq \gamma \mathbf{I}.$$

However, examining the structure of the inner product $\mathbf{Y}_{N,d}^T \mathbf{Y}_{N,d}$ reveals
$$\resizebox{\columnwidth}{!}{%
$\mathbf{Y}_{N,d}^T(t) \mathbf{Y}_{N,d}(t) = \left[\begin{array}{cccc}\mathbf{Y}_o^T \mathbf{Y}_o & \mathbf{Y}_o^T \mathbf{Y}_r & \cdots & \mathbf{Y}_o^T \mathbf{Y}_r \\ \mathbf{Y}_r^T \mathbf{Y}_o & \mathbf{Y}_r^T \mathbf{Y}_r & \cdots & \mathbf{Y}_r^T\mathbf{Y}_r \\ \vdots & &\ddots & \vdots \\ \mathbf{Y}_r^T \mathbf{Y}_o &\mathbf{Y}_r^T \mathbf{Y}_r &\cdots &\mathbf{Y}_r^T \mathbf{Y}_r\end{array}\right],$}$$
which we note has the same two (block) rows repeated $\mathcal{F}$ times. Further, we note the integral $$\int_t^{t+T} \mathbf{Y}_{N,d}^T(\tau) \mathbf{Y}_{N,d}(\tau)d\tau$$ will have the same block structure, meaning the integrated matrix must be rank deficient. 

Thus, regardless of the reference signal $\mathbf{q}_d$, $\mathbf{Y}_{N,d}$ cannot be persistently exciting; therefore we can make no claim on the asymptotic behavior of the individual parameter errors $\widetilde{\mathbf{a}}_i$, except that they remain bounded, and their derivative $\dot{\widetilde{\mathbf{a}}}_i$ converges to zero.

However, we can also write the closed-loop dynamics as
\begin{align*}
    \mathbf{H}(\mathbf{q})\dot{\mathbf{s}} + (\mathbf{C}+\mathbf{K}_D)\mathbf{s} &= \textstyle \sum_{i=1}^{N} \mathbf{Y}_1 \widetilde{\mathbf{a}}_i,\\
    &= N \mathbf{Y}_1 \left( \textstyle \frac{1}{N} \sum_{i=1}^{N} \widetilde{\mathbf{a}}_i \right),
\end{align*}
which implies, if $\mathbf{Y}_1$ is persistently exciting, that the average parameter error $\frac{1}{N} \sum_{i=1}^{N} \widetilde{\mathbf{a}}_i$ converges to zero. 

Thus, we can interpret the lack of a persistently exciting signal $\mathbf{q}_d$ as a reflection of the redundancy inherent in the system; since there exists a subspace of individual wrenches $\bm{\tau}_i$ which produce the same total wrench $\bm{\tau}$ on the body, even a signal which would be persistently exciting for a single agent can only guarantee that the average parameter error vanishes.

Interestingly, this is in many ways dual to the results in \cite{Wensing:2018aa}, wherein the authors consider the problem of adaptive control for cloud-based robots. In this case, the authors have a large number of decoupled, but identical, manipulators, which they aim to control adaptively using a shared parameter estimate. The authors find that parameter convergence is guaranteed if the sum of the regressor matrices is persistently exciting, a much weaker condition than any single robot executing a persistenly exciting trajectory. On the other hand, in this setting, we have a single physical system which is controlled collaboratively by many robots using individual parameter estimates; we find that in this case persistent excitation requires not just a stronger condition on $\mathbf{q}_d$, but is in fact impossible to achieve.
\subsection{Friction and Dissipative Dynamics}
\label{subsec:friction}
Finally, in some systems (such as the planar manipulation task), there are non-conservative forces which act on the payload due to effects such as friction. The simplest model treats these forces as an additive body-frame disturbance $\mathbf{D}(\mathbf{q})\dot{\mathbf{q}}$ which is linear in the velocity of the measurement point, where $\mathbf{D}(\mathbf{q}) = 
\mathbf{R}\bm{\Lambda}_D\mathbf{R}^T$, for a constant positive semidefinite matrix of friction coefficients $\bm{\Lambda}_D$. It can be easily verified that such a model is appropriate for viscous sliding friction. 

This friction can be compensated quite easily using an additional term in the control law, namely
$$\mathbf{F}_i = \mathbf{Y}_o \hat{\mathbf{o}}_i + \mathbf{Y}_f \hat{\mathbf{f}}_i - \mathbf{K}_D \mathbf{s},$$
where $\mathbf{Y}_f(\mathbf{q},\dot{\mathbf{q}}_r) \mathbf{f}_i = \alpha_i \mathbf{D}(\mathbf{q}) \dot{\mathbf{q}}_r$. If we additionally use the adaptation law $$\dot{\hat{\mathbf{f}}}_i = -\bm{\Gamma}_f\mathbf{Y}_f^T(\mathbf{q},\dot{\mathbf{q}}_r)\mathbf{s},$$
with positive definite $\bm{\Gamma}_D$, one can show this results in $$\dot{V} = -\mathbf{s}^T \left(N\mathbf{K}_D + \mathbf{D}(\mathbf{q})\right)\mathbf{s},$$
recalling that $\mathbf{D}\mathbf{s} = \mathbf{D}\dot{\mathbf{q}}-\mathbf{D}\dot{\mathbf{q}}_r.$ Thus, in effect, we have increased the PD gain with an additional positive semi-definite term $\mathbf{D}(\mathbf{q})$, yielding a higher ``effective'' $\mathbf{K}_D$ matrix without increasing the gains directly.

However, since such an expression is not obtainable in general, e.g. in the case of Coulomb friction, we turn instead to the case when frictional forces are modeled as being applied at the attachment point of each agent. In this case, the frictional forces are also pre-multiplied by the grasp matrices, resulting in the dynamics $$\mathbf{H}\ddot{\mathbf{q}} + \mathbf{C}\dot{\mathbf{q}}+ \mathbf{g} = \textstyle \sum_{i=1}^{N} \mathbf{M}_i \left(\bm{\tau}_i - \mathbf{D} \mathbf{v}_i - \mathbf{D}_C \sgn( \mathbf{v}_i)\right),$$
where $\mathbf{v}_i = \mathbf{M}_i^T \dot{\mathbf{q}}$ is the linear and angular velocity of the robot in the global frame, $\mathbf{D}$ and $\mathbf{D}_C$ are positive semidefinite matrices of friction coefficients, and $\sgn(\cdot)$ is the element-wise sign function.

\begin{table}
    \centering
    \ra{1.3}
    \resizebox{0.45\columnwidth}{!}{%
    \begin{tabular}{@{}ll@{}}\toprule
        \textbf{parameter} & \textbf{value}\\
        \midrule
        $m$ & \num{1.89e4} \si{kg}\\
        $I_{xx}$ & \num{1.54e4} \si{kg.m^2}\\
        $I_{yy}$ & \num{1.54e4} \si{kg.m^2}\\
        $I_{zz}$ & \num{2.37e3} \si{kg.m^2}\\
        $\mathbf{r}_p$ & $\left[0, 0, 1.5\right]$ \si{m}\\
        $\mathbf{r}_1$ & $\left[0, 0, 1.5\right]$ \si{m}\\
        $\mathbf{r}_2$ & $\left[0, 0, -1.5\right]$ \si{m}\\
        $\mathbf{r}_3$ & $\left[0.5, 0, 0\right]$ \si{m}\\
        $\mathbf{r}_4$ & $\left[-0.5, 0, 0\right]$ \si{m}\\
        $\mathbf{r}_5$ & $\left[0, 0.5, 0\right]$ \si{m}\\
        $\mathbf{r}_5$ & $\left[0, -0.5, 0\right]$ \si{m}\\
        \bottomrule\\
    \end{tabular}}
    \caption{Parameter Values for Simulation in $SE(3)$}
    \label{tab:params}
\end{table}

These terms can again be compensated, now by augmenting the ``shifted'' control law \eqref{eq:ctrl_1},
$$\bm{\tau}_i = \widehat{\mathbf{M}}_i^{-1}\mathbf{F}_i + \mathbf{Y}_{D} \hat{\mathbf{d}}_i + \mathbf{Y}_c \hat{\mathbf{c}}_i ,$$
where we have $$\mathbf{Y}_D(\mathbf{q},\dot{\mathbf{q}}_r)\mathbf{d}_i = \mathbf{M}_i \mathbf{D} \mathbf{M}_i^T \dot{\mathbf{q}}_r,$$ and $$\mathbf{Y}_C(\mathbf{q},\mathbf{v}_i) \mathbf{c}_i = \mathbf{M}_i \mathbf{D}_C \sgn(\mathbf{v}_i),$$ which can be computed assuming each agent can measure its own linear and angular rate $\mathbf{v}_i$.

Using this new control law, along with adaptation laws
\begin{align*}
    \dot{\hat{\mathbf{d}}}_i &= -\bm{\Gamma}_D \mathbf{Y}_D^T \mathbf{s},\\
    \dot{\hat{\mathbf{c}}}_i &= -\bm{\Gamma}_C \mathbf{Y}_C^T \mathbf{s},
\end{align*}
results in $$\dot{V} = -\mathbf{s}^T\left(N\mathbf{K}_D + \mathbf{M}_i\mathbf{D}\mathbf{M}_i^T\right)\mathbf{s},$$ which again adds a positive definite term $\mathbf{M}_i \mathbf{D} \mathbf{M}_i^T$ to  $\dot{V}$, using the viscous friction applied at each point. We note that the Coulomb friction must be cancelled exactly, due to the presence of unknown parameters (namely $\mathbf{r}_i$) in the $\sgn(\cdot)$ nonlinearity.
\section{Simulation Results}

We simulated the performance of the proposed controller for a collaborative manipulation task in $SE(3)$ to verify the previous theoretical results. To this end, we constructed a simulation scenario where a group of $N = 6$ space robots are controlling a large, cylindrical rigid body (such as a rocket body) to track a desired periodic trajectory $\mathbf{q}_d(t)$.

Table \ref{tab:params} lists the parameters used in the simulation, along with their values. 
The payload was a cylinder with a height of \num{1.5} \si{m} and radius \num{0.5} \si{m}. The agents were placed symmetrically along the principal axes of the body, with the measurement point $\mathbf{r}_p$ collocated with the first agent, $\mathbf{r}_1$, replicating the scenario when one agent broadcasts its measurements to the group.
The reference trajectory $\mathbf{q}_d$ had desired linear and angular rates $\dot{\mathbf{q}}_d$ which were sums of sinusoids, e.g. $$v_{x,d}(t) = \sum_{i=1}^{n_f} \alpha_i \cos(\omega_i t),$$
with the number of frequencies $n_f = 5$. The desired pose $\mathbf{q}_d$ started from an initial state $\mathbf{q}_d(0)$ and forward-integrated the dynamics according to the desired velocity.

The simulations were implemented in Julia 1.1.0. All continuous dynamics were discretized and numerically integrated using Heun's method \cite{Stoer:1980aa}, a two-stage Runge-Kutta method, using a step size of $h=$\num{1e-2} \si{s}. All forward integration steps involving rotation matrices were implemented using the exponential map $$\mathbf{R}_{t+h} = \exp\left(h \bm{\omega}_t^\times\right) \mathbf{R}_t,$$ which ensures that $\mathbf{R}$ remains in $SO(3)$ for all time. Since the proposed adaptation laws are designed for continuous-time systems, we modify them by adding a deadband in which no adaptation will occur. For all simulations, we stop adaptation when $\lvert\lvert \mathbf{s} \rvert \rvert_2 \leq 0.01$. This technique is well-studied, and is commonly used to improve the robustness of adaptive controllers to unmodeled dynamics.

The physical parameters of the rigid body are given in Table \ref{tab:params}. The gain matrix $\bm{\Gamma}_o$ was set to \num{0.3}$*\diag(\lvert\mathbf{a}_o\rvert + 0.01)$, where, for a vector $\mathbf{v} \in \mathbb{R}^n$, $\diag(\mathbf{v})$ returns the matrix in $\mathbb{R}^{n \times n}$ with $\mathbf{v}$ on its diagonal, and zeros everywhere else. We place the correct parameters on the gain diagonal to scale the adaptation gains correctly, since the parameters range in magnitude from \num{e-1} to \num{e4}. For $\bm{\Gamma}_r$, we use \num{e-3}$*\mathbf{I}$. Finally, we let $\lambda = 1.5$, and $\mathbf{K}_D = \diag(\mathbf{K}_F, \mathbf{K}_M)$, where $\mathbf{K}_F =$ \num{5e4}$*\mathbf{I}$, and $\mathbf{K}_M =$ \num{5e3}$*\mathbf{I}$. The initial parameter estimates $\widehat{\mathbf{o}}_i$ and $\widehat{\mathbf{r}}_i$ were drawn from zero-mean multivariate Gaussian distributions with variances of $\Sigma_o = \mathbf{I}$ and $\Sigma_r = 2\mathbf{I},$ respectively. Source code of the simulation is publicly available \href{https://github.com/pculbertson/hamilton_ac}{here}.

\begin{figure}[t]
    \centering
    \subfloat[Simulated error signals for task in $SE(3)$.\label{subfig:se3_err}]{
        \includegraphics[width=0.4\textwidth]{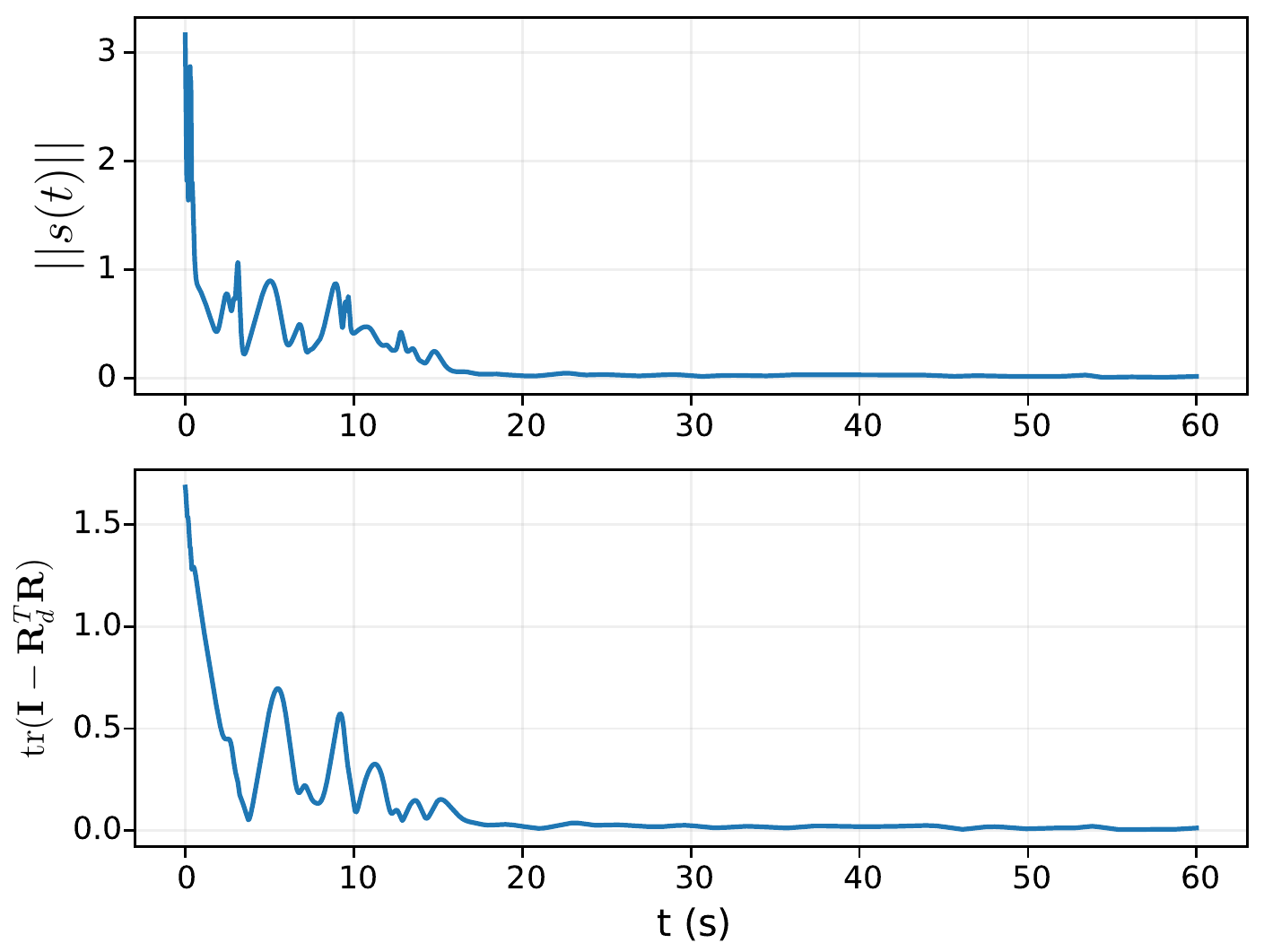}
    }\\
    \subfloat[Simulated Lyapunov-like function $V(t)$ for $SE(3)$ control.\label{subfig:se3_lyap}]{
        \includegraphics[width=0.35\textwidth]{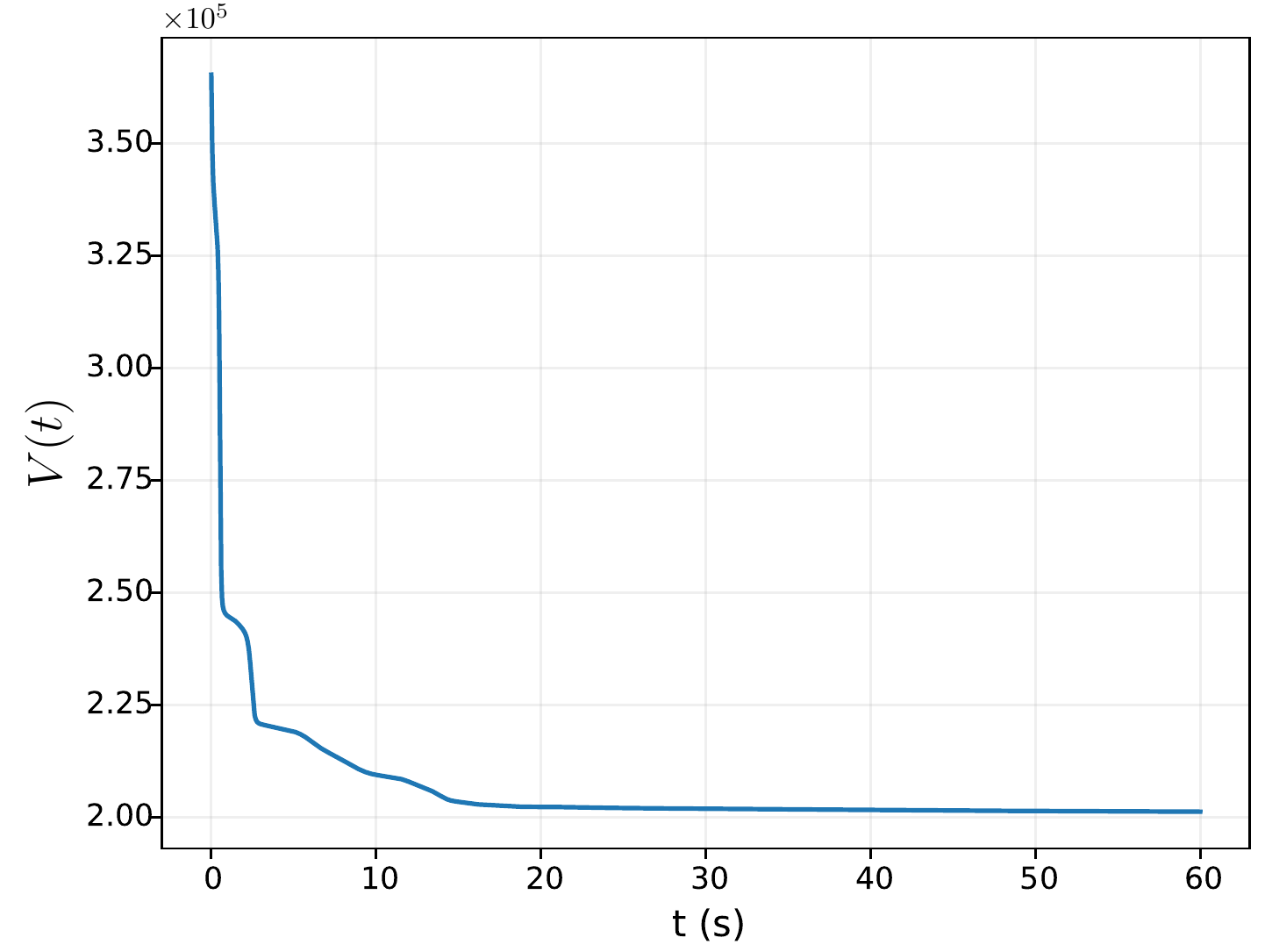}
    }
    \caption{Simulation results for a group of $N = 6$ agents controlling a body in $SE(3)$. \textbf{\textit{(a)}}: Simulated values of the composite error norm $\lvert \lvert \mathbf{s}(t) \rvert \rvert$ and rotation error. We notice both the composite and rotation errors converge to zero. \textbf{\textit{(b)}}: Simulated value of the Lyapunov-like function $V(t)$. We notice it is non-increasing (i.e., $\dot{V}(t) \leq 0$), and converges to a constant, non-zero value, since the parameter errors do not vanish.}
    \label{fig:se(3)}
\end{figure}

\begin{figure}[ht]
    \centering
    \includegraphics[width=0.45\textwidth]{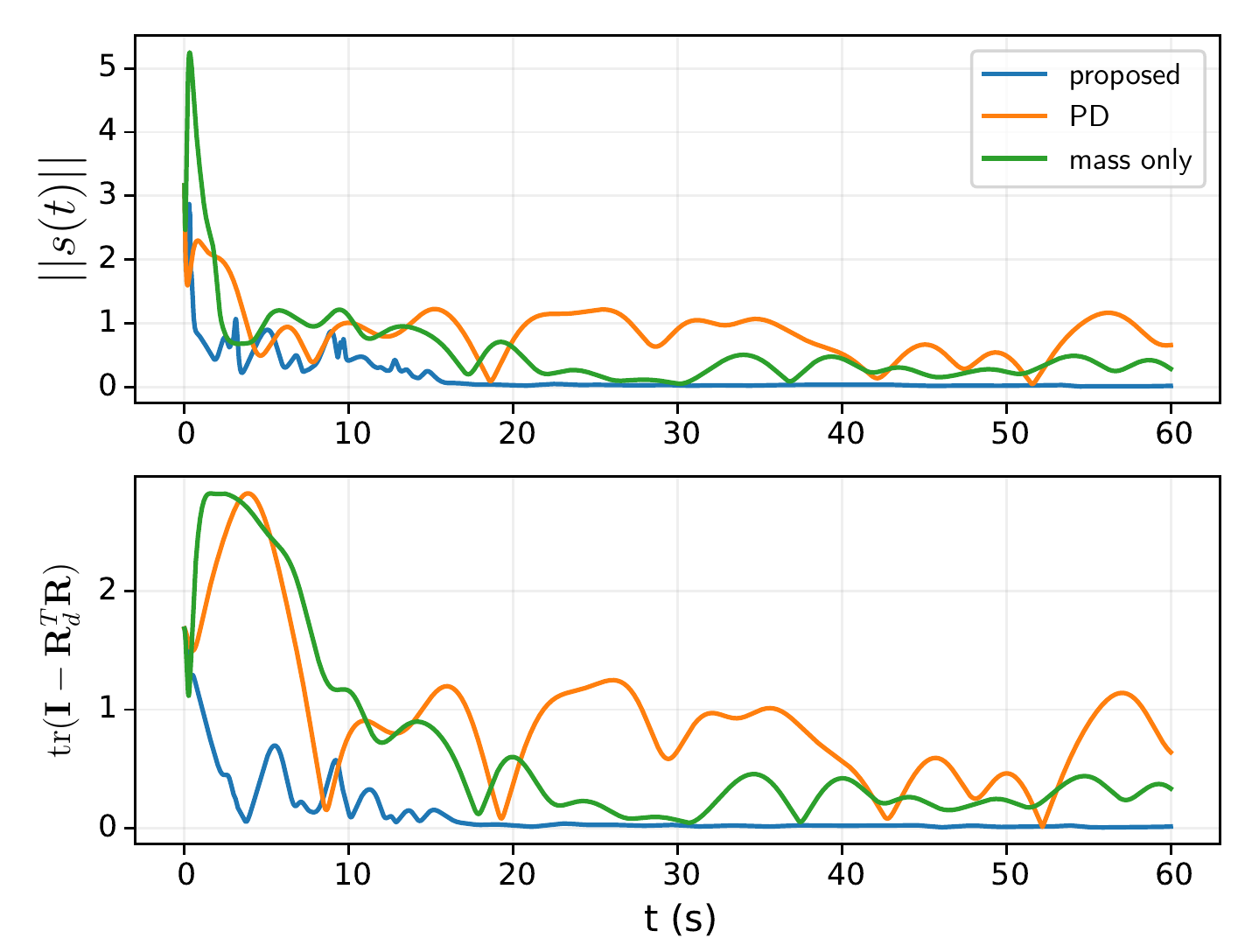}
    \caption{Simulated baseline comparisons for the $SE(3)$ manipulation task. In green, we plot the error signals in the case where the agents estimate and adapt to only the object parameters $\widehat{\mathbf{o}}_i$, as in \cite{Liu:1998aa, Verginis:2017aa}. In this case, neither the composite error $\mathbf{s}$ nor the rotation error converges to zero, even if simulated over a long horizon. The controller may also diverge for some initial conditions. In orange, the agents implement a PD controller by setting their parameter estimates $\widehat{\mathbf{o}}_i, \widehat{\mathbf{r}}_i$ to zero for all time. We note the error does not converge to zero, and tracking performance is generally poor.}
    \label{fig:baselines}
\end{figure}

Figure \ref{fig:se(3)} plots the error signals and Lyapunov-like function value over the course of the simulation. We can see that the composite error quickly converges to zero. We also notice the rotation error converges to zero, implying $\mathbf{R} \rightarrow \mathbf{R}_d$, as expected. We can further see that the Lyapunov-like function $V(t)$ is strictly decreasing (implying $\dot{V}(t) \leq 0$), and converges to a constant, non-zero value. This implies that the parameter errors $\widetilde{\mathbf{o}}_i, \widetilde{\mathbf{r}}_i$ do not converge to zero, which matches intuition due to the lack of a persistently exciting signal, as shown in Section \ref{subsec:pe}.

We further compared the controller's performance against some simple baselines; namely, we studied the controller's performance when agents only adapt to the common object parameters $\mathbf{o}_i$, as studied in \cite{Liu:1998aa, Verginis:2017aa}, and in the case of no adaptation, which reduces to the case of a simple PD controller, since all parameter values are initialized to zero. The error signals for both cases are plotted in Figure \ref{fig:baselines}. In the first case ($\bm{\Gamma}_r = \bm{0}$), we can see that that controller performance greatly suffers due to the inaccurate moment arm estimates. Specifically, in this case, the composite error $\mathbf{s} \not\rightarrow \bm{0}$; even over a very long time horizon, we found the error never vanishes. We also found that simply doubling the variance of the distribution from which the moment arm initializations $\widehat{\mathbf{r}}_i$ were sampled, i.e., sampling $\widehat{\mathbf{r}}_i \sim \mathcal{N}(\bm{0},2\sqrt{2}\mathbf{I})$, would usually cause the controller to diverge. 

This behavior is as expected, since when the moment arm estimates are poor, there is a large, unmodelled coupling between the rotational and translational dynamics which controller cannot compensate. In the second case, we plot the performance which results from only using the PD term in the control law \eqref{eq:ctrl_2}, to demonstrate the value of the proposed adaptation routine. We can clearly see that the PD term alone is unable to send the composite error $\mathbf{s}$ to zero; thus the proposed adaptation is essential to ensure that the tracking error vanishes.

We were also interested in how the controller performs in the case of single-agent failure (i.e., if one or more agents are deactivated at some time $t_d$ during the manipulation task). We note that this is an instantaneous shift in the system parameters (since new $\alpha_i$ must be chosen such that the constants still sum to one). This is, in turn, equivalent to restarting the controller with new system parameters $\mathbf{o}_i$, and a different initial condition; thus the boundedness and convergence results still hold. Figure \ref{fig:drop} plots the system behavior in the case that three agents (half the manipulation team) are deactivated at $t_d= 30$ \si{s}. As expected, when the agents are dropped, we see a discontinuity in the Lyapunov-like function $V(t)$, due to the parameter change, and a slight jump in the composite error $\mathbf{s}$. However, the remaining three agents are still able to regain tracking, as expected.
\begin{figure}
    \centering
    \subfloat[Simulated error signals, $3$ agents shut off at $t=30$.\label{subfig:drop_err}]{
        \includegraphics[width=0.42\textwidth]{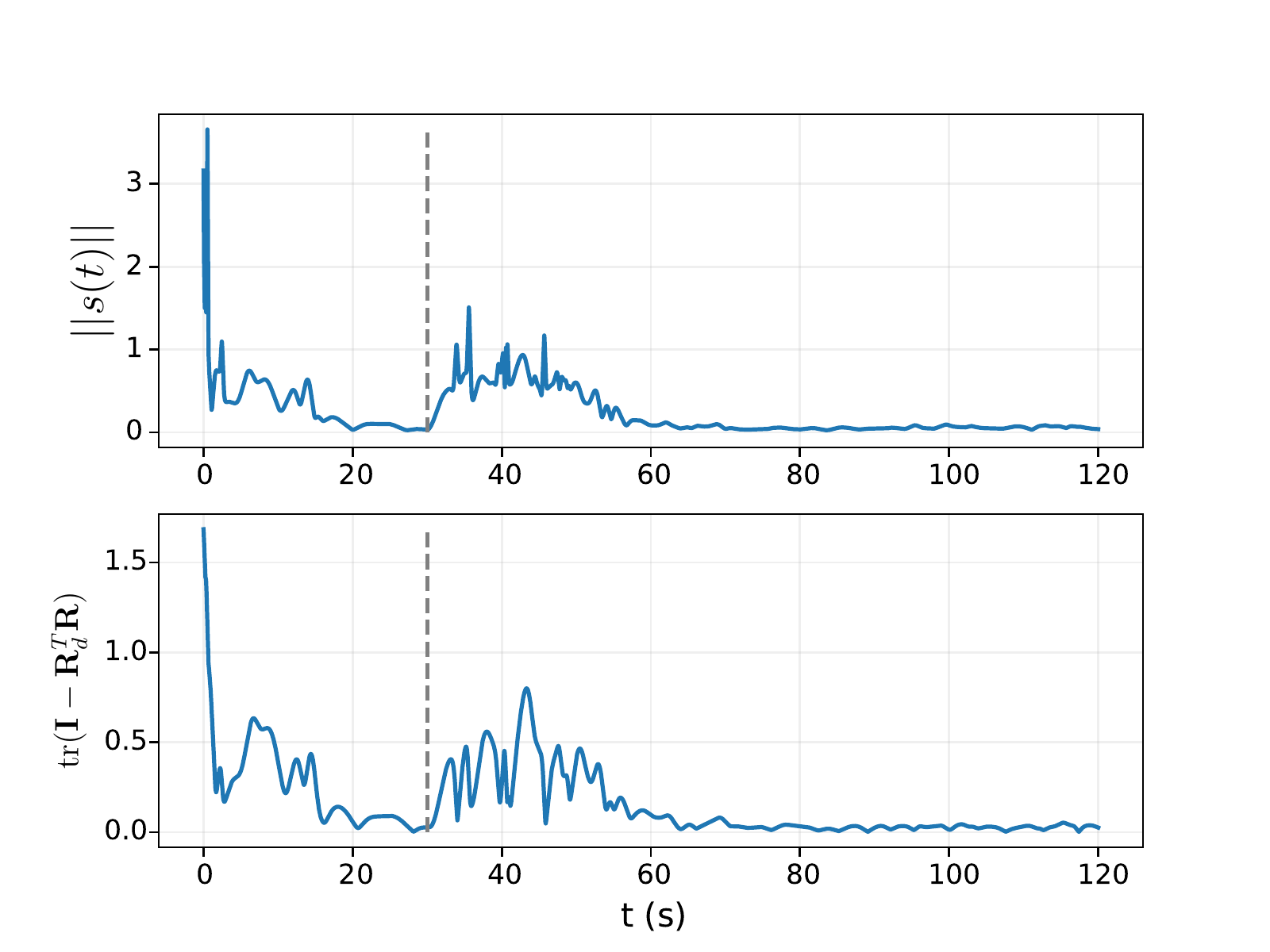}
    }\\
    \subfloat[Simulated Lyapunov-like function during agent shutoff.\label{subfig:drop_lyap}]{
        \includegraphics[width=0.37\textwidth]{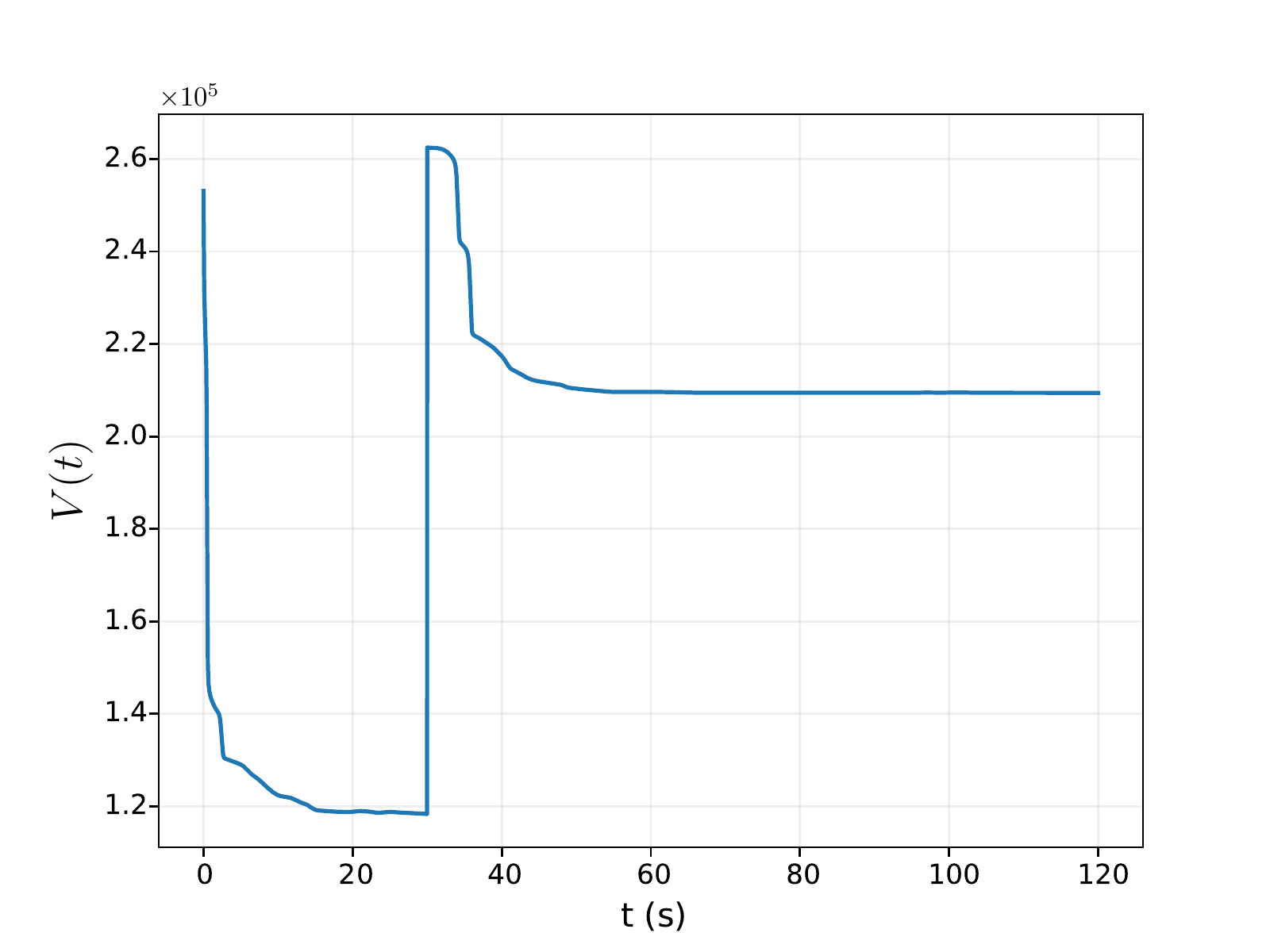}
    }
    \caption{Simulation results for $SE(3)$ task, where 3 of the $N=6$ agents are deactivated at time $t=30$. \textbf{\textit{(a)}}: Simulated values of the composite error norm $\lvert \lvert \mathbf{s}(t) \rvert \rvert$ and rotation error. The ``drop time'' is shown with a dashed grey line. After the drop, the remaining agents return to tracking the reference trajectory. \textbf{\textit{(b)}}: Simulated value of the Lyapunov-like function $V(t)$. We note the function is non-increasing, except for the discontinuity at $t=30$, which occurs due to the instantaneous switch of parameter values.}
    \label{fig:drop}
\end{figure}

\begin{figure}[ht]
    \centering
    
        \includegraphics[width=0.4\textwidth]{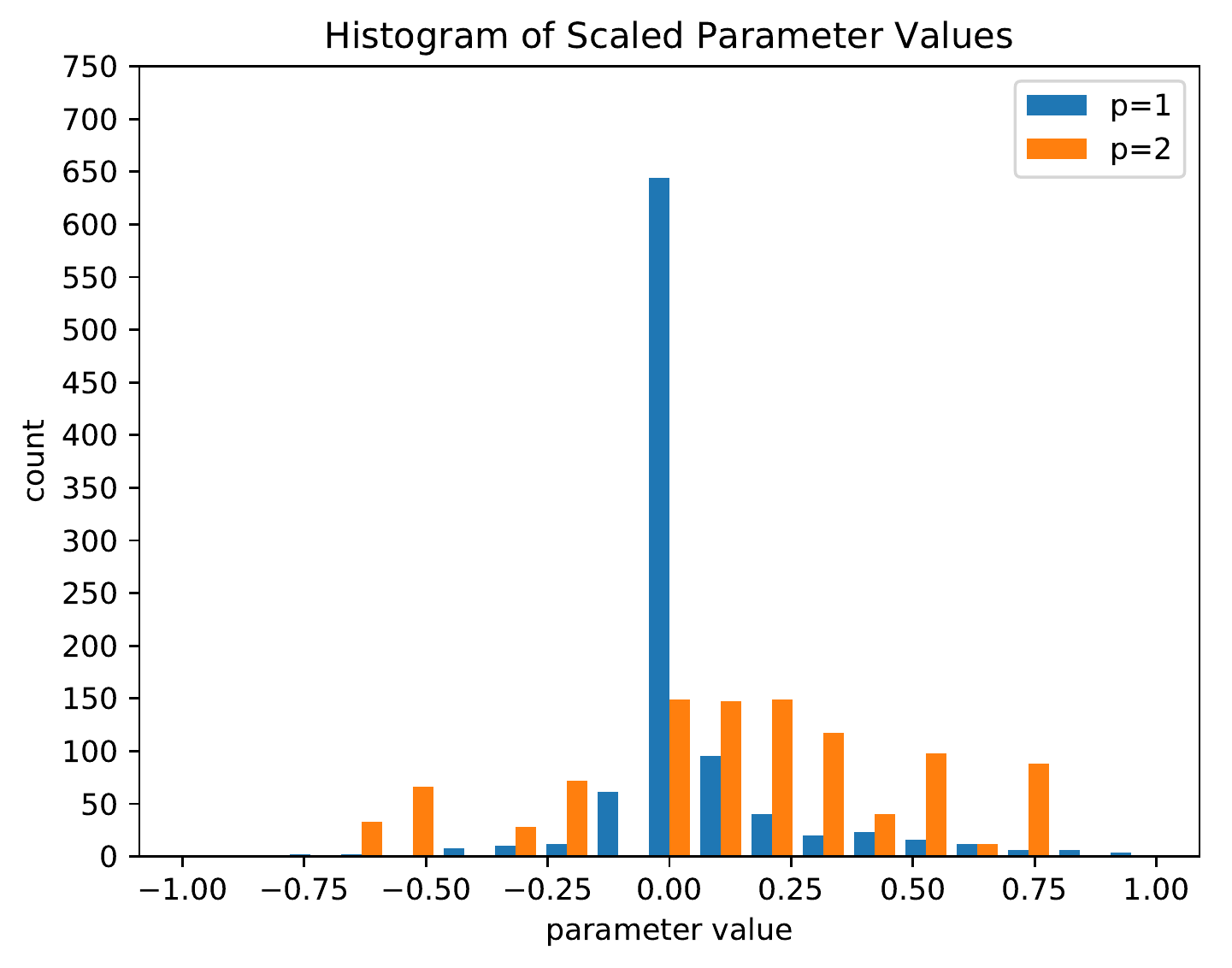}
    \caption{Histogram of estimated parameter values for $\ell_1$ and $\ell_2$ regularization. When using the $\ell_1$ norm (blue), we see the resulting parameter estimate is quite sparse, with a large spike centered at 0. When using the $\ell_2$ norm (orange), we instead see the distribution of parameter estimates is less peaked, but takes on fewer large values.}
    \label{fig:param_hist}
\end{figure}

\begin{figure*}[t]
    \centering
    \subfloat{
        \includegraphics[width=0.85\textwidth]{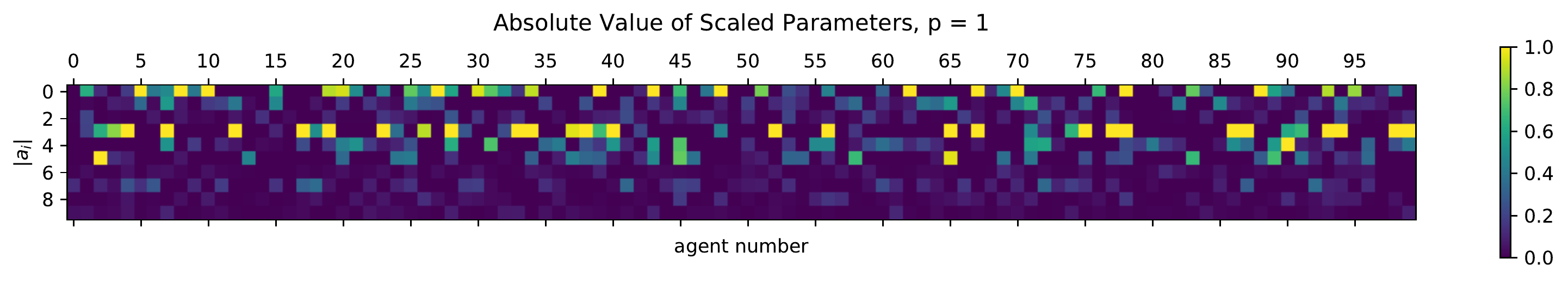}
    }\\
    \subfloat[]{
        \includegraphics[width=0.85\textwidth]{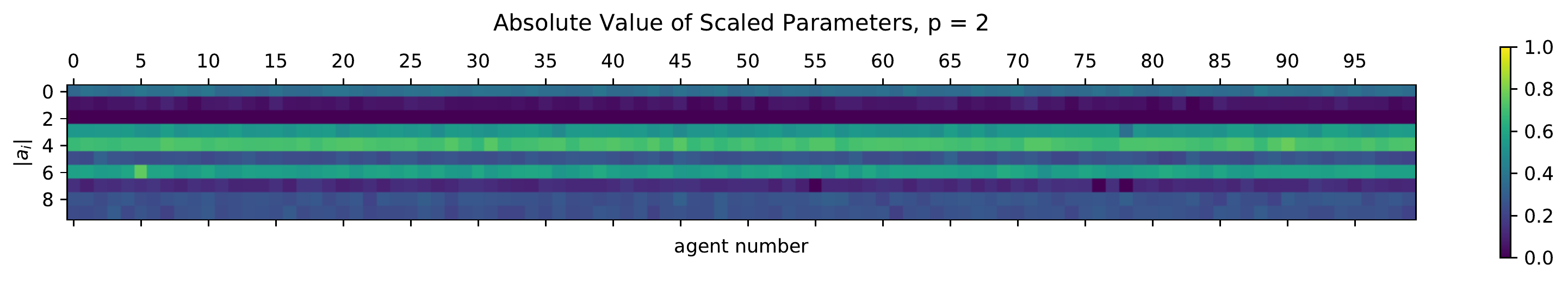}
    }
    \caption{Visualization of estimated parameters for each regularization term. $\textbf{Top:}$ Visualization of the estimated parameters when using the Bregman divergence for the $\ell_1$ norm. In this case, we can see this regularization does indeed promote sparsity, with many entries being equal or close to zero, with a few outliers taking on larger absolute values. $\textbf{Bottom:}$ Visualization of the estimated parameters when using the $\ell_2$ norm. In this case, the parameters take on much smaller values, but are roughly equal across all agents.}
    \label{fig:param_vis}
    \vspace{-1.5em}
\end{figure*}

Finally, as shown recently in \cite{Lee:2018ab, Wensing:2020}, the parameter error terms in the Lyapunov-like function \eqref{eq:lyap}, $\widetilde{\mathbf{a}}_i^T \bm{\Gamma}_i^{-1} \widetilde{\mathbf{a}}_i$ can be generalized by instead using the Bregman divergence $d_\psi(\bm{\Gamma}_i\mathbf{a}_i \parallel \bm{\Gamma}_i \widehat{\mathbf{a}}_i),$ where, for some strictly convex function $\psi,$ we write the Bregman divergence as $$d_\psi(\mathbf{x} \parallel \mathbf{y}) = \psi(\mathbf{y}) - \psi(\mathbf{x}) - (\mathbf{y}-\mathbf{x})^T \nabla \psi(\mathbf{x}).$$ If we use the adaptation law $$\dot{\widehat{\mathbf{a}}}_i = -\bm{\Gamma}_i^{-1} (\nabla^2 \psi(\bm{\Gamma}_i \hat{\mathbf{a}}))^{-1} \bm{\Gamma}_i^{-1} \mathbf{Y}_i^T \mathbf{s},$$ we can easily reproduce the boundedness and convergence results shown previously. In fact, one can show the usual quadratic penalty is simply a Bregman divergence for $\psi(\mathbf{x}) = \frac{1}{2}\lvert\lvert \mathbf{\Gamma}^{-\frac{1}{2}}\mathbf{x} \rvert \rvert_2^2.$

% By replacing the quadratic term with the more general Bregman divergence, we now have an adaptation law which uses a ``natural gradient'' with respect to the Hessian metric of the convex function $\psi$. The idea of using a natural gradient for adaptation, in order to impose underlying Riemannian constraints on the parameters, was introduced in \cite{Lee:2018ab}, and generalized by \cite{Wensing:2018ac}. In \cite{Boffi:2019aa}, the authors explore this idea further, and show that the choice of $\psi$ actually allows different forms of regularization to be imposed on the parameter estimates. This is especially important due to the inherent overparameterization of our system. 

As a cursory study of this behavior, we substituted the usual quadratic penalty for a Bregman divergence with $\psi(\mathbf{x}) = \lvert \lvert \mathbf{x} \rvert \rvert_1$, the $\ell_1$ norm. In \cite{Boffi:2021}, the authors show the $\ell_1$ norm leads to sparse solutions, meaning the parameter estimates converge to vectors which can achieve tracking with few non-zero values. 

% Since our system is highly overparameterized, we were interested in comparing this sparse solution with the one achieved using the typical Euclidean geometry.

For this experiment, we increased the number of agents to $N=100$ to exaggerate the overparameterization of the system. While the tracking performance and transient error were quite similar for both regularization strategies, they give rise to quite different estimates of the object parameters. Figure \ref{fig:param_vis} plots histograms of the estimated parameters under both the $\ell_1$ norm and $\ell_2$ norm divergences. Here the parameters are scaled so all values are on the same order. We can indeed see the $\ell_1$ solution is much sparser, with many values equal to zero, and with more outliers than the $\ell_2$ case.

Figure \ref{fig:param_vis} shows a visualization of the estimated parameters. We can again see that the $\ell_1$ solution is quite sparse, with many entries equal to zero, in addition to the few outliers which are have relatively large values. In contrast, the $\ell_2$ solution shows a ``banded'' structure that implies agreement between agents, where all agents arrive at parameter estimates that are somewhat small in magnitude, and roughly equal across agents. 

Again, while these experiments are quite cursory, they demonstrate that, due to the overparameterization of the system, we can design various forms of regularization to express preferences about which parameter estimates, among those that achieve tracking, are desirable. The question of how various regularization terms affect the system's robustness to noise, or the overall control effort, could prove interesting directions for future research.

\section{Experimental Results}
We further studied the performance of the proposed controller experimentally, using a group of $N = 6$ ground robots performing a manipulation task in the plane. Figure \ref{fig:ouijabot} shows one of the robots used in this experiment, a Ouijabot \cite{Wang:2018ab}, which is a holonomic ground robot designed specifically for collaborative manipulation. The robots were equipped with current sensors on each motor, which allows them to estimate and control the wrench they are applying to a payload. Specifically, the robots used a low-level PID controller to generate motor torques, based on the desired wrench set by the higher-level adaptive controller, with feedback from the current sensors. They are also equipped with a Raspberry Pi 3B+, which allows them to perform onboard computation, as well as receive state measurements via ROS. The robots did not use any other communication capabilities.

\begin{figure}
    \centering
    \includegraphics[width=0.4\textwidth]{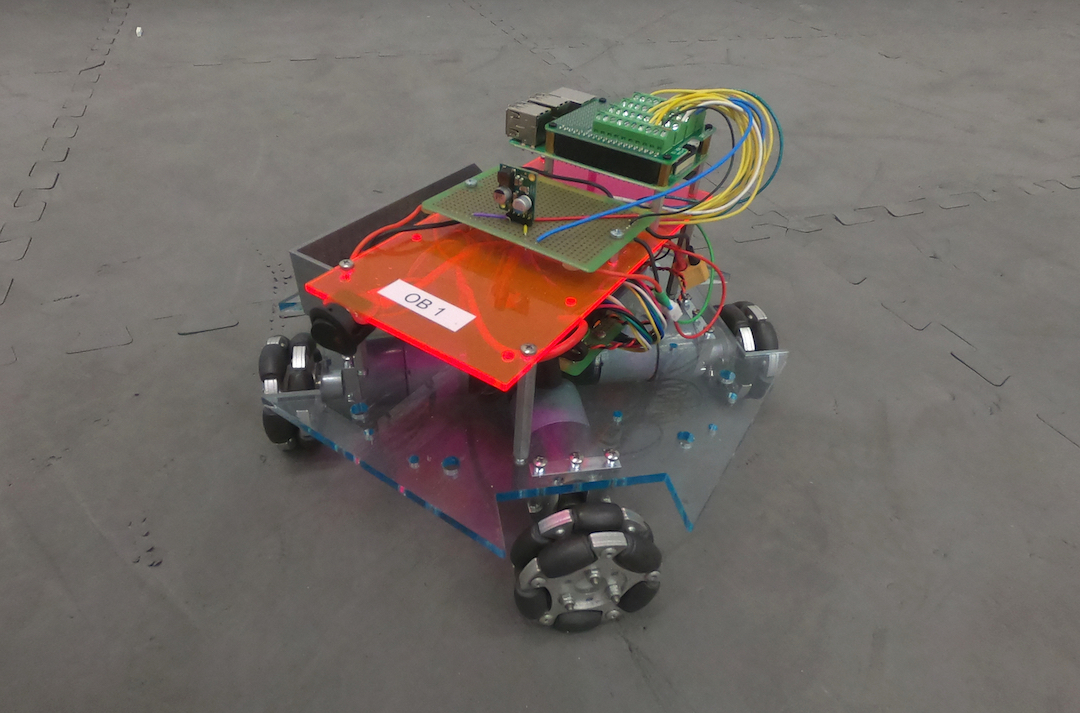}
    \caption{Photo of a Ouijabot, the holonomic ground robot used to conduct the experiments presented in this work. The robot is equipped with current sensors on its motors, which allows it to close a feedback loop to provide a desired force and torque to an external payload.}
    \label{fig:ouijabot}
\end{figure}

Figure \ref{fig:payload} shows the payload which the robots manipulated, which was an asymmetric wood piece that weighed roughly $5.5$ \si{kg}, which was too heavy for any individual agent to move. The robots were rigidly attached to the piece, which remained roughly 4 \si{cm} from the ground. As seen in the figure, the measurement point $\mathbf{r}_p$ coincided with one of the robots, and was located quite far from the center of mass, resulting in long moment arms $\mathbf{r}_i$ as well as non-negligible coupling between the rotational and translational dynamics. The robots received position measurements from this point using Optitrack, a motion capture system; the velocity measurements were computed by numerically differentiating the position measurements. All signals were passed through a first-order filter to reduce the effects of measurement noise.

\begin{figure}
    \centering
    \includegraphics[width=0.4\textwidth]{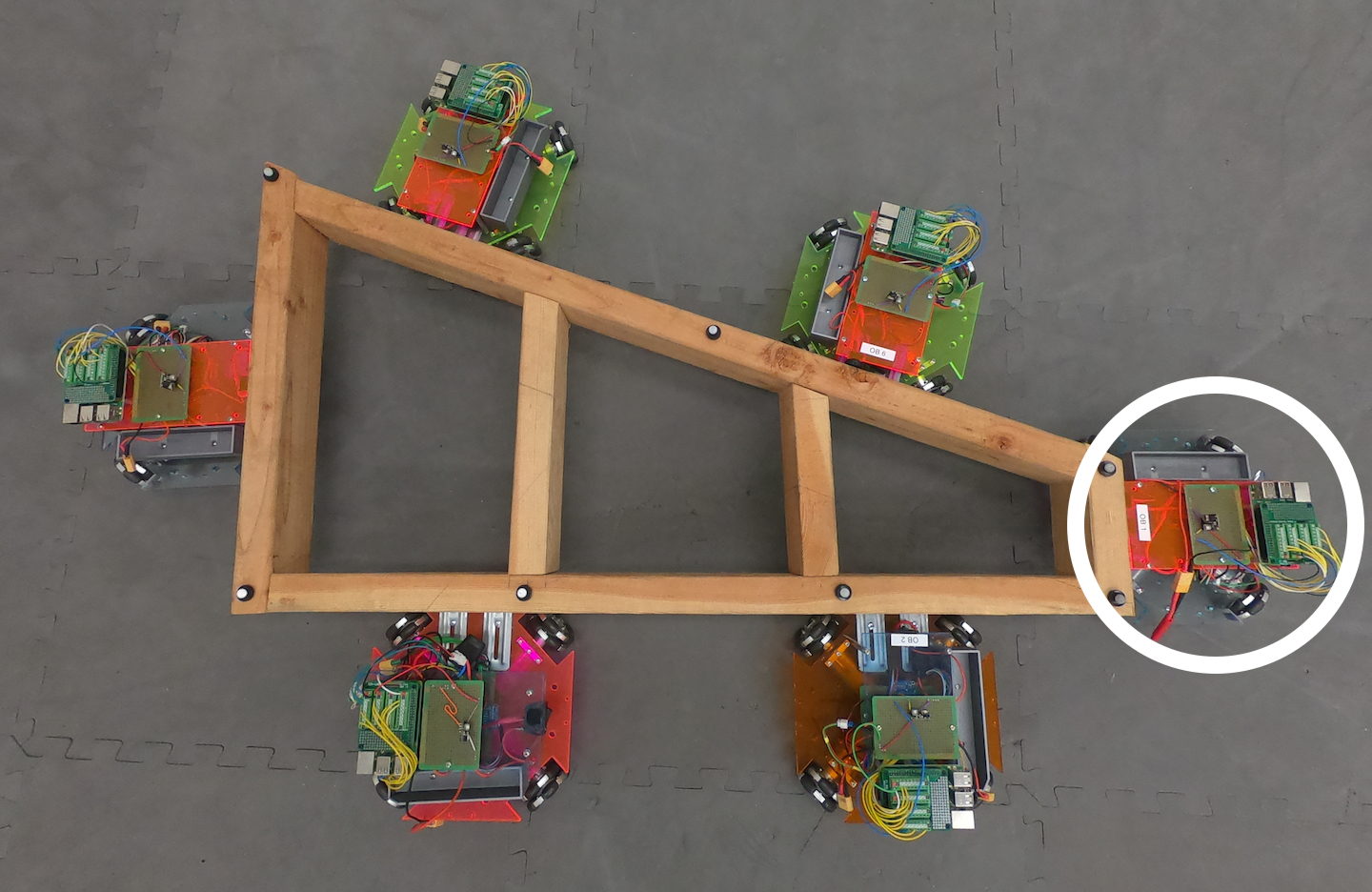}
    \caption{Experimental setup for planar manipulation task. A team of $N = 6$ ground robots manipulate a shared wood payload. The measurement point for the payload coincides with one of the agents, circled in white, and is located far from the geometric center of the body, and its center of mass.}
    \label{fig:payload}
\end{figure}

\begin{figure*}[ht]
    \centering
    \hfill\subfloat[Composite error norm, $\lvert\lvert \mathbf{s} \rvert\rvert$, averaged over 25 trials.\label{subfig:exp_err}]{
        \includegraphics[width=0.45\textwidth]{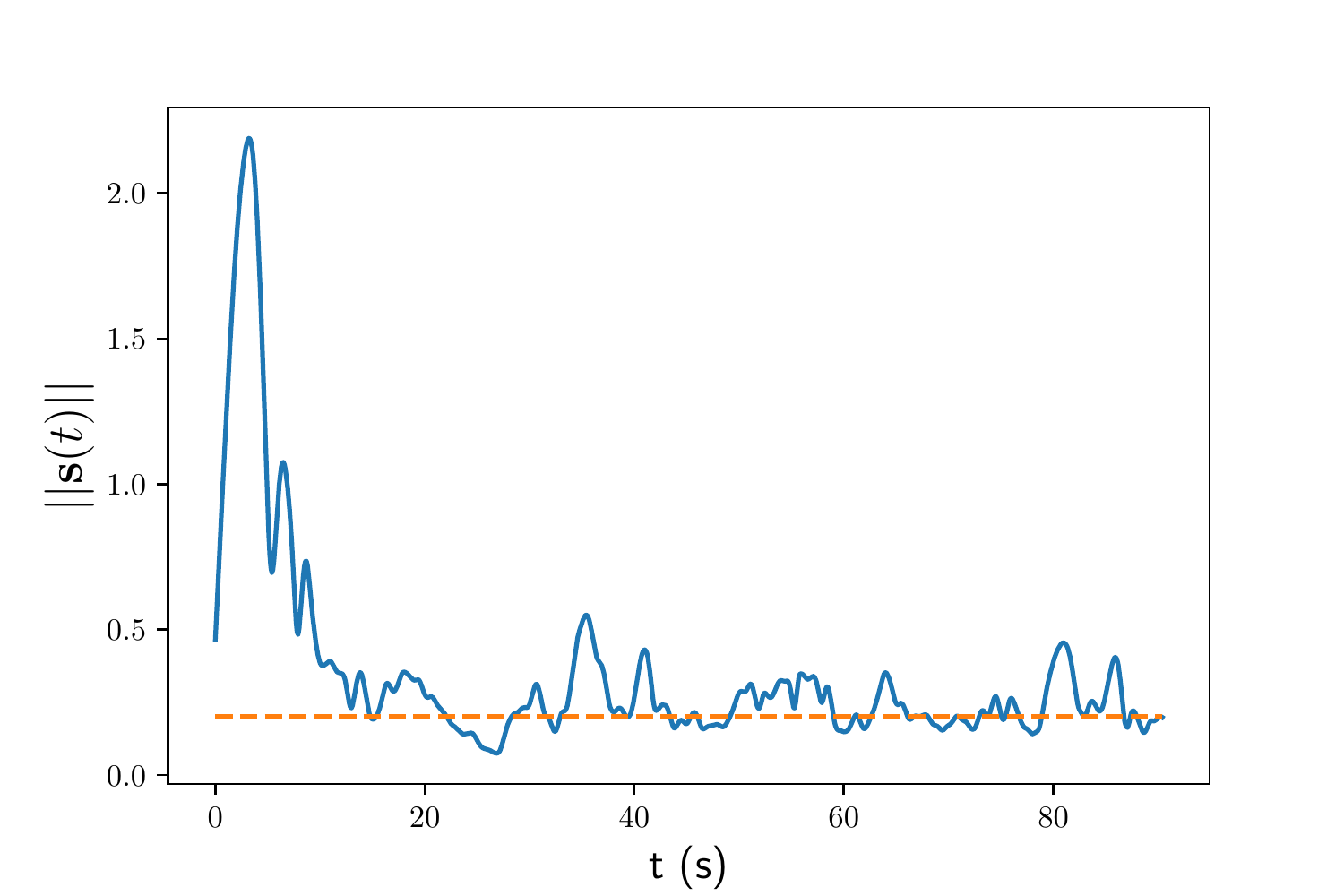}
    }
    \hfill
    \subfloat[Composite error norm, plotted with PD baseline.\label{subfig:exp_baseline}]{
        \includegraphics[width=0.45\textwidth]{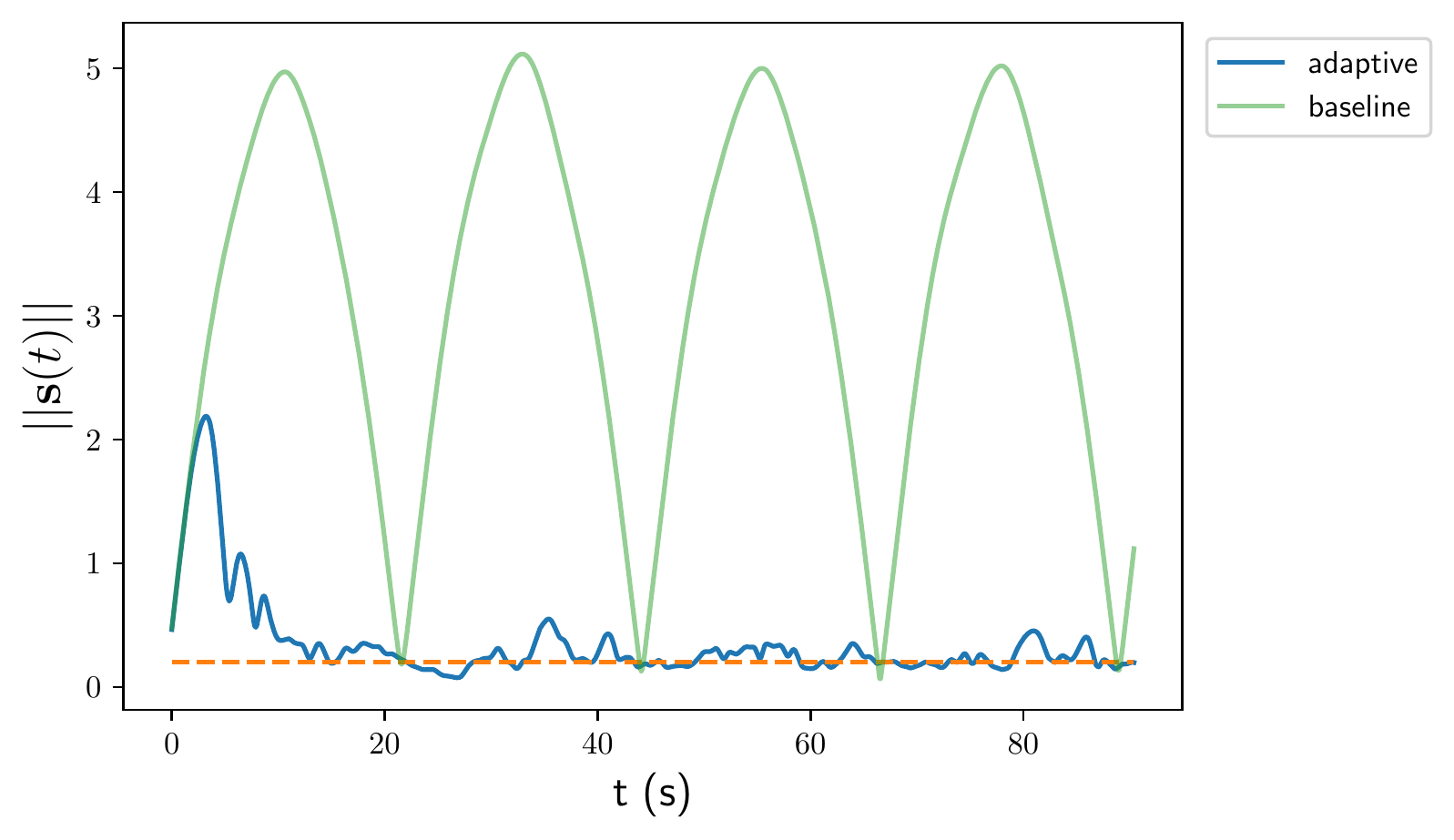}
    }
    \hfill
    \\
    \hfill
    \subfloat[Ground tracks for all trials.\label{subfig:all_tracks}]{
        \includegraphics[width=0.45\textwidth]{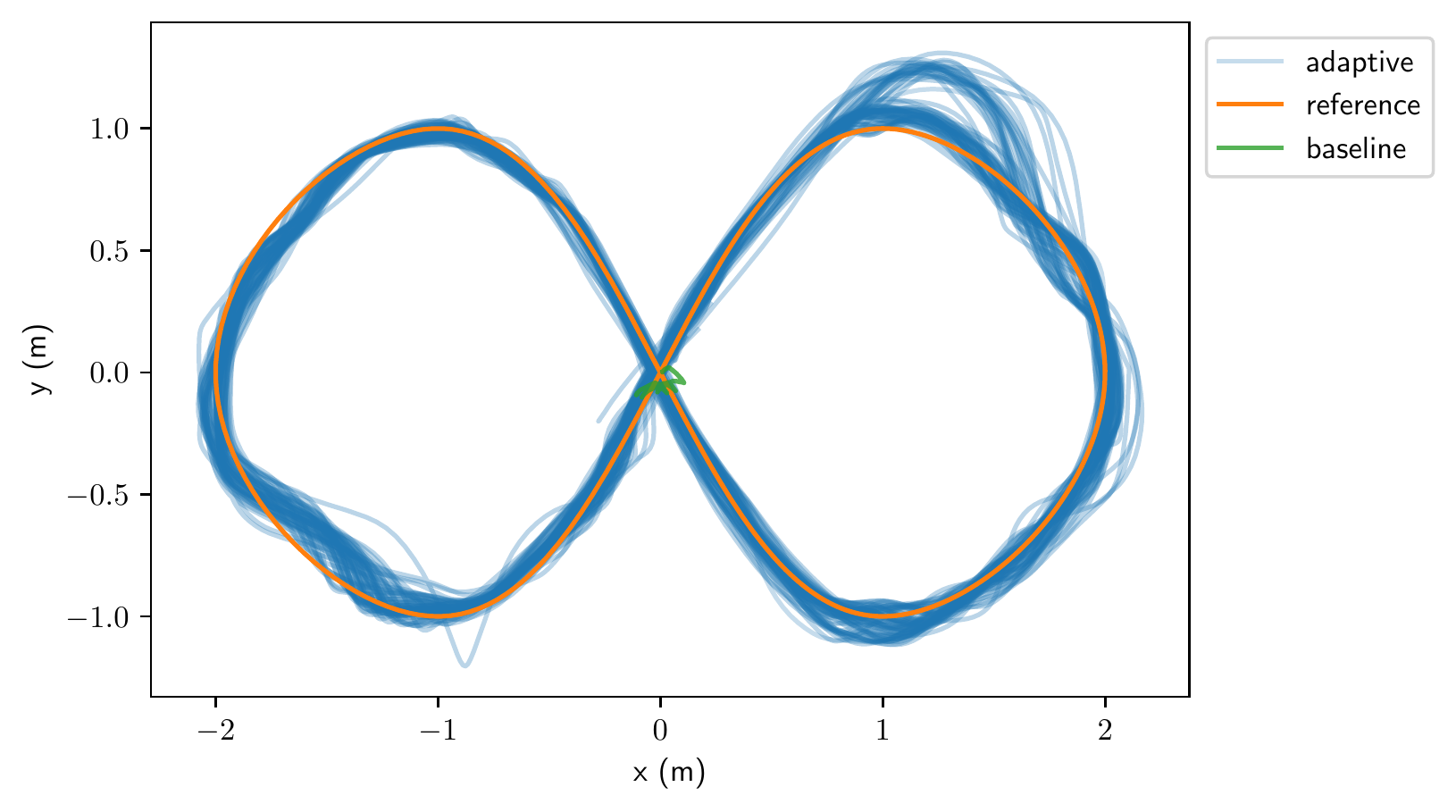}
    }
    \hfill
    \subfloat[Average payload heading, over 25 trials.\label{subfig:heading}]{
        \includegraphics[width=0.45\textwidth]{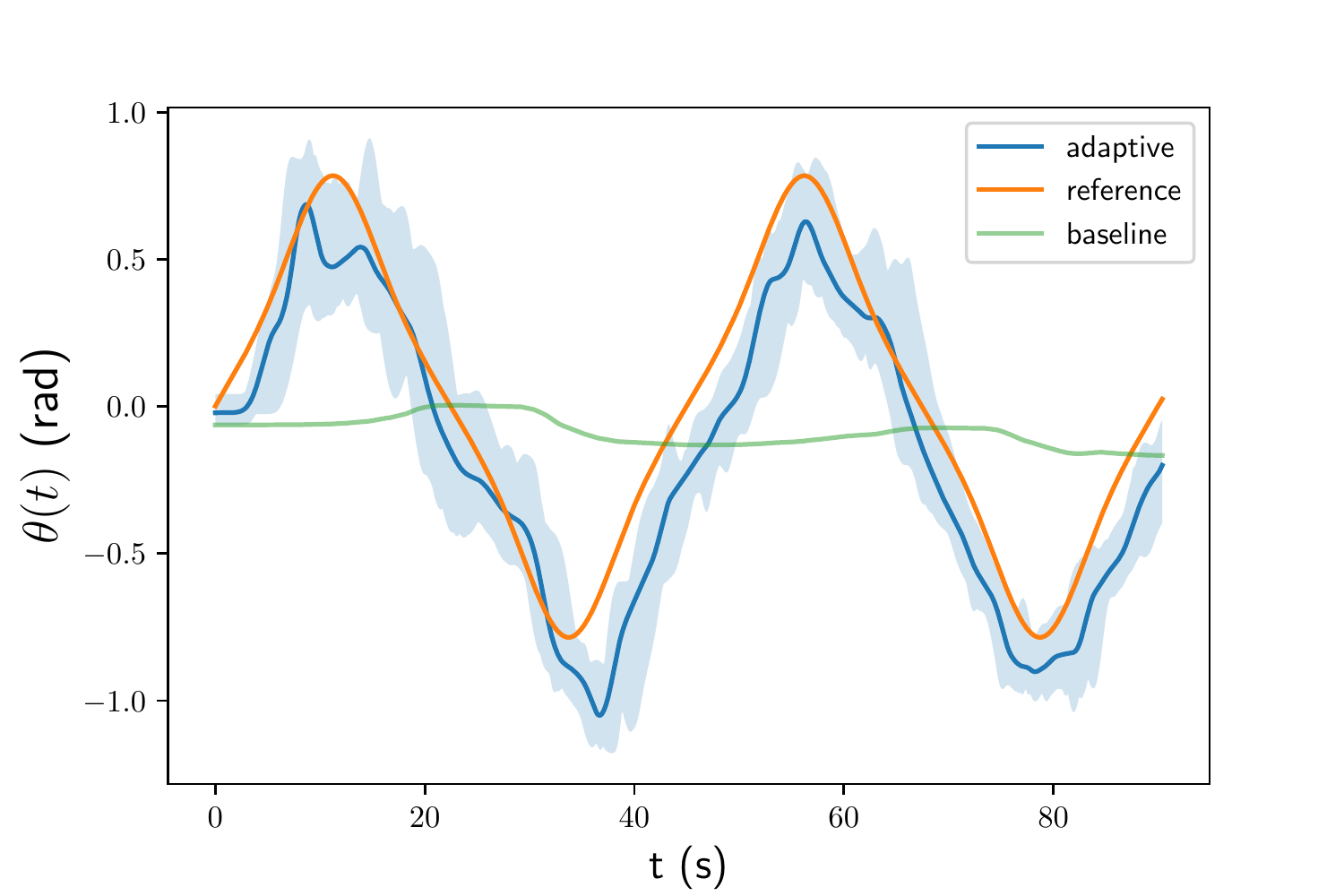}
    }
    \hfill
    \caption{Experimental results for a group of $N = 6$ ground robots performing a planar manipulation task. \textbf{\textit{(a)}}: Measured value of the composite error, averaged across 25 trials. \textbf{\textit{(b)}} Composite error, plotted with the values measured when only using the PD term of the control law. \textbf{\textit{(c)}}: Overhead view of the measurement point trajectory for all trials. \textbf{\textit{(d)}}: Plot of average payload angle (blue), with shading between maximum and minimum values across all trials. Plotted with reference trajectory (orange), and PD baseline response (green).}
    \label{fig:experiment}
    \vspace{-1em}
\end{figure*}

Figure \ref{fig:experiment} plots the results of the experimental study. The desired trajectory $\mathbf{q}_d$ of the measurement point corresponded to a ``figure-8'' in the plane, with the desired rotation $\theta_d$ oscillating between $-\frac{\pi}{4}$ and $\frac{\pi}{4}$ at successive waypoints. The experiment included 25 trials of the manipulation task, with all initial parameter estimates set to zero. Figure \ref{subfig:exp_err} plots the composite error norm, averaged over the 25 trials, with the deadband of $0.2$, under which adaptation was stopped, plotted as a dotted line. We can see that the controller converges to be within the deadband in about 45 seconds. Figure \ref{subfig:all_tracks} shows the ground track of all trials, with the reference shown in orange, and baseline (PD) response shown in green. We note the agents typically overshoot the first turn (top right), and then converge to reasonable tracking, although it proved difficult to achieve perfect tracking through all turns. Figure \ref{subfig:heading} plots the heading of the body, averaged over all trials; shaded is the maximum and minimum heading value at each time. We again note the  controller achieves reasonable tracking after a short transient, which improves over the trial; the PD controller alone gave nearly no response.

% and the response of the underlying PD controller plotted in green. The initial state of the robot is the center of the ``figure-8,'' from which it moves up and to the left into the first turn of the desired trajectory. A single experimental run consisted of three cycles of the reference signal, which ran over roughly 100 \si{s}. We purposely used quite low PD gains in order to demonstrate the effectiveness of the adaptation routine; the PD control alone gives essentially no response to the reference trajectory.
We note this experiment introduced considerable sources of modelling error which were not parameterized by the adaptive controller. Specifically, the wrenches computed by the high-level controller were only setpoints for a low-level PID controller. Further, while the control law included terms for both viscous and Coulomb friction, as outlined in Section \ref{subsec:friction}, this of course did not capture the complex frictional relationship between the robot wheels and foam floor of the experimental space, which depended on wheel speed, direction, etc. Finally, all control signals were subject to time delays, discretization error, and measurement noise. We found our controllers (which were implemented as ROS nodes) ran at a frequency of 20 \si{Hz}.

\section{Concluding Remarks}
This work has focused on the problem of collaborative manipulation without knowledge of payload parameters or centralized control. To this end, we proposed a feedback control architecture on $SE(3),$ together with a distributed adaptation law, which enabled the agents to gradually adapt to the unknown object parameters to track a desired trajectory. We further provided a Lyapunov-like proof of the asymptotic convergence of the system to the reference trajectory under our proposed controller, and studied its performance for practical manipulation tasks using both numerical simulations and hardware experiments.

There exist numerous directions for future work. An important open question is how to incorporate actuation constraints, such as saturation, into the proposed control framework. In fact, tight actuation constraints typically are a strong motivation for collaborative manipulation, but in this work we assume the agents may apply any desired force and torque to the object. This is because, in general, it is quite difficult to incorporate saturation into adaptive control design. While our hardware experiments demonstrate that the proposed controller can achieve good performance despite these unmodelled actuation constraints, handling these theoretically remains an important open problem.

Another interesting direction is to understand more clearly the connections between this work, and other distributed control architectures such as quorum sensing or the control of networked systems. As mentioned previously, studying the effect of different regularization strategies, or of using higher-order adaptation laws, on the overall system performance could also prove valuable. Further, while the proposed controller is robust to slow changes in the object's mass and geometric properties, another interesting direction would be to study how and if the proposed architecture can be extended to the manipulation of articulated or flexible bodies. Finally, while we assume the agents travel through free space, we are quite interested in enabling distributed manipulation in the presence of obstacles, and studying how agents can coordinate themselves through object motion to avoid obstacles they cannot directly observe.

\bibliographystyle{unsrt} 
\balance
\bibliography{main}

\appendix
\section{Proofs of Matrix Properties}   
\label{sec:matrix_properties}
Here we demonstrate some properties of $\mathbf{H}, \mathbf{C}$, and their derivatives, which we use to prove Theorem \ref{thm:controller}. These properties are well-known for general Lagrangian equations of motion \cite[Chapter~9]{Slotine:1991aa}; here we prove them for our choices of $\mathbf{H}, \mathbf{C}.$

% We first show the positive definiteness of $\mathbf{H}$. Physically, since the kinetic energy of the system can be expressed as $$K = \frac{1}{2} \dot{\mathbf{q}}^T \mathbf{H}(\mathbf{q}) \dot{\mathbf{q}},$$ we must have $\mathbf{H}(\mathbf{q}) \succ 0,$ otherwise there exists $\dot{\mathbf{q}} \neq \bm{0}$ with $K = 0.$ 

\begin{lemma}
$\mathbf{H}(\mathbf{q}),$ as given in \eqref{eq:H_SO(3)}, is positive definite.
\end{lemma}
\begin{proof}
For a block matrix of the form
\begin{equation*}
    \mathbf{M} = \left[\begin{array}{cc}\mathbf{A} & \mathbf{B} \\ \mathbf{B}^T & \mathbf{C} \end{array} \right],
\end{equation*}
using the properties of the Schur complement \cite{Boyd:2004aa}, we can say $\mathbf{M} \succ 0$ if and only if $\mathbf{A} \succ 0$  and $\mathcal{S} := \mathbf{C} - \mathbf{B}^T \mathbf{A}^{-1} \mathbf{B} \succ 0.$ Thus, since $m\mathbf{I} \succ 0$ by inspection, $\mathbf{H} \succ 0$ if and only if $\mathcal{S} = \mathbf{R}\mathbf{I}_p\mathbf{R}^T + m(\mathbf{R}\mathbf{r}_p)^\times (\mathbf{R}\mathbf{r}_p)^\times \succ 0.$

We can further use the fact that for $\mathbf{v} \in \mathbb{R}^3,$
\begin{equation*}
    \mathbf{v}^{\times} \mathbf{v}^{\times} = \mathbf{v} \mathbf{v}^T - (\mathbf{v}^T \mathbf{v}) \mathbf{I},
\end{equation*}
to write 
\begin{align*}
    \mathcal{S} &= \mathbf{R}\left(\mathbf{I}_{cm} \pm m(\mathbf{r}_p^T \mathbf{r}_p - \mathbf{r}_p \mathbf{r}_p^T)\right)\mathbf{R}^T, \\
    &= \mathbf{R}\mathbf{I}_{cm}\mathbf{R}^T,
\end{align*}
so we have $\mathcal{S} \succ 0,$ since $\mathbf{I}_{cm} \succ 0$ for bodies of finite size. Thus, we conclude $\mathbf{H} \succ 0.$

\end{proof}

% We can further show $\dot{\mathbf{H}} - 2\mathbf{C}$ is skew-symmetric. Specifically, using the principle of conservation of energy, we can write $$\dot{\mathbf{q}}^T (\bm{\tau}-\mathbf{g}) = \frac{1}{2}\frac{d}{dt} \left[ \dot{\mathbf{q}}^T \mathbf{H}(\mathbf{q}) \dot{\mathbf{q}}\right].$$
% Differentiating the right side yields $$\textstyle \dot{\mathbf{q}}^T (\bm{\tau}-\mathbf{g}) = \dot{\mathbf{q}}^T\mathbf{H}\ddot{\mathbf{q}} + \frac{1}{2}\dot{\mathbf{q}}^T \dot{\mathbf{H}}\dot{\mathbf{q}},$$
% which, rearranging, results in $$\dot{\mathbf{q}}^T (\dot{\mathbf{H}}-2\mathbf{C}) \dot{\mathbf{q}} = \bm{0},$$
% which must hold true for all $\dot{\mathbf{q}}$. This may be understood as a matrix expression of the conservation of. energy. While the matrix $\mathbf{C}$ is non-unique, there exist choices of $\mathbf{C}$ where $\dot{\mathbf{H}}-2\mathbf{C}$ is skew-symmetric. We now demonstrate this property for our choice of $\mathbf{C}$.

\begin{lemma}
    $\dot{\mathbf{H}} - 2 \mathbf{C}$ is skew symmetric.
\end{lemma}

\begin{proof}

$\dot{\mathbf{H}}$ is given by
\begin{equation*}
    \dot{\mathbf{H}} = \left[\begin{array}{cc} \bm{0}_3 & \mathbf{A}-\mathbf{A}^T  \\ \mathbf{A}^T-\mathbf{A}  & \bm{\omega}^\times \mathbf{R}\mathbf{I}_p\mathbf{R}^T - \mathbf{R}\mathbf{I}_p\mathbf{R} \bm{\omega}^\times \end{array}\right],
\end{equation*}
where $\mathbf{A} = m \bm{\omega}^\times(\mathbf{R}\mathbf{r}_p)^\times$, and using the fact $\dot{\mathbf{R}} = \bm{\omega}^\times \mathbf{R}$. Further,
\begin{equation*}
    \dot{\mathbf{H}}-2\mathbf{C} = \left[\begin{array}{cc} \bm{0}_3 & 
    -\mathbf{B} \\ \mathbf{B} & \mathbf{D}
    \end{array}\right],
\end{equation*}
where $\mathbf{B} = \mathbf{B}^T =  \mathbf{A} + \mathbf{A}^T$ and $\mathbf{D} = -\bm{\omega}^\times \mathbf{R}\mathbf{I}_p\mathbf{R}^T - \mathbf{R}\mathbf{I}_p\mathbf{R} \bm{\omega}^\times - m((\mathbf{R}\mathbf{r}_p)^\times \dot{\mathbf{x}})^\times$. We can see $\dot{\mathbf{H}} - 2\mathbf{C}$ is skew-symmetric, since $\mathbf{D}$ is skew-symmetric, and the off-diagonal blocks are symmetric and of opposite sign.
\end{proof}

\begin{IEEEbiography}[{\includegraphics[width=1in,height=1.25in,clip, keepaspectratio]{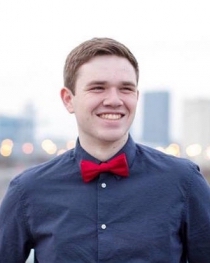}}]{Preston Culbertson} received a BS from the Georgia Institute of Technology in 2016, and an MS in 2020 from Stanford University, where he is currently a PhD candidate Mechanical Engineering. His research interests include multi-robot manipulation and grasping, distributed optimization and control, and learning and adaptation in robotic systems, especially among groups of interacting agents. 
\end{IEEEbiography}

\begin{IEEEbiography}
    [{\includegraphics[width=1in,height=1.25in,clip,keepaspectratio]{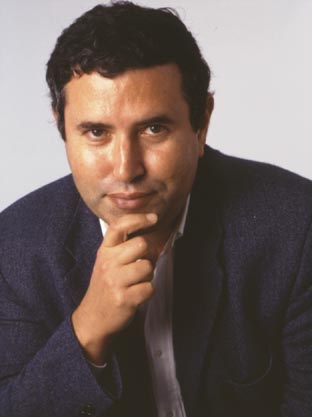}}]  {Jean-Jacques Slotine} is Professor of Mechanical Engineering and Information Sciences, Professor of Brain and Cognitive Sciences, and Director of the Nonlinear Systems Laboratory at~MIT. He received his Ph.D. from~MIT in 1983. After working at Bell Labs in the computer research department, he joined the MIT faculty in 1984. His  research focuses on developing rigorous but practical tools for nonlinear systems analysis and control. These have included key advances and experimental demonstrations in the contexts of sliding control, adaptive nonlinear control, adaptive robotics, machine learning, and contraction analysis of nonlinear dynamical systems. Professor Slotine is the co-author of two graduate textbooks, ``Robot Analysis and Control'' (Asada and Slotine, Wiley, 1986), and ``Applied Nonlinear Control'' (Slotine and Li, Prentice-Hall, 1991). He was a member of the French National Science Council from 1997 to 2002 and of Singapore’s A*STAR SigN Advisory Board from 2007 to 2010. He currently is a member of the Scientific Advisory Board of the Italian Institute of Technology and a Distinguished Visiting Faculty at Google AI. He is the recipient of the 2016 Oldenburger Award.
\end{IEEEbiography}

\begin{IEEEbiography}[{\includegraphics[width=1in,height=1.25in,clip]{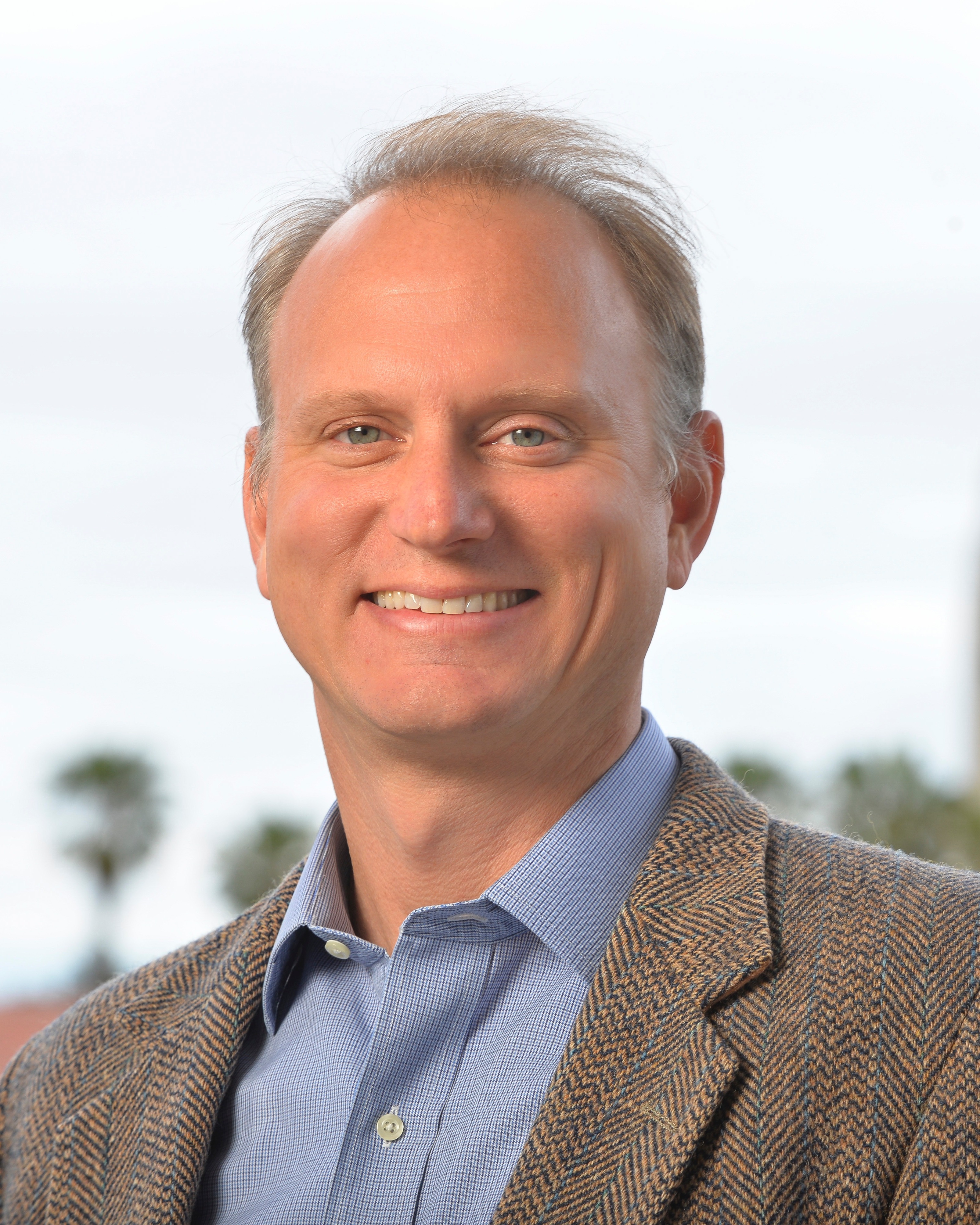}}]
{Mac Schwager} is an assistant professor with the Aeronautics and Astronautics Department at Stanford University.  He obtained his BS degree in 2000 from Stanford University, his MS degree from MIT in 2005, and his PhD degree from MIT in 2009.  He was a postdoctoral researcher working jointly in the GRASP lab at the University of Pennsylvania and CSAIL at MIT from 2010 to 2012, and was an assistant professor at Boston University from 2012 to 2015.  He received the NSF CAREER award in 2014, the DARPA YFA in 2018, and a Google faculty research award in 2018, and the IROS Toshio Fukuda Young Professional Award in 2019.  His research interests are in distributed algorithms for control, perception, and learning in groups of robots, and models of cooperation and competition in groups of engineered and natural agents.
\end{IEEEbiography}

\end{document}